%% file: neurips_2026.tex
\documentclass{article}

\PassOptionsToPackage{numbers,sort&compress}{natbib}
\usepackage[preprint]{neurips_2026}

\usepackage[utf8]{inputenc} % allow utf-8 input
\usepackage[T1]{fontenc}    % use 8-bit T1 fonts
\usepackage[colorlinks=true,citecolor=blue,urlcolor=blue,linkcolor=black]{hyperref}       % hyperlinks
\usepackage{xurl}           % better line breaks for long URLs
\usepackage{graphicx}       % \includegraphics
\usepackage{booktabs}       % professional-quality tables
\usepackage{multirow}        % multi-row cells in tables
\usepackage{float}          % provides the [t] float specifier
\usepackage{amsmath}        % math environments, \eqref, \lVert\cdot\rVert
\usepackage{amsthm}         % theorem styles/environments (\theoremstyle, \newtheorem)
\usepackage{amssymb}        % additional math symbols
\usepackage{amsfonts}       % blackboard math symbols
\usepackage{nicefrac}       % compact symbols for 1/2, etc.
\usepackage{microtype}      % microtypography
\usepackage{xcolor}         % colors
\usepackage{enumitem}       % customize list spacing/margins
\usepackage{subcaption}
\usepackage{tikz}
\usetikzlibrary{arrows.meta,positioning,calc,fit}
\floatstyle{ruled}
\newfloat{algorithm}{tbp}{loa}
\floatname{algorithm}{Algorithm}
\floatstyle{plain}

\newtheorem{definition}{Definition}
\newtheorem{assumption}{Assumption}
\newtheorem{theorem}{Theorem}
\newtheorem{lemma}{Lemma}

\newtheorem{corollary}{Corollary}
\newtheoremstyle{myremark}% name
  {\topsep}% space above
  {\topsep}% space below
  {\normalfont}% body font
  {}% indent amount
  {\bfseries}% head font
  {.}% punctuation after head
  {.5em}% space after head
  {}% head spec
\theoremstyle{myremark}
\newtheorem{remark}{Remark}[section]
\newcommand{\CB}{\mathrm{CB}}
\newcommand{\spl}{\mathrm{softplus}}
\newcommand{\Pl}{\mathcal{P}_{\mathrm{legit}}}
\newcommand{\Pf}{\mathcal{P}_{\mathrm{fraud}}}
\newcommand{\Qr}{\mathcal{Q}_{\mathrm{rand}}}
\newcommand{\bv}{\mathbf{v}}
\newcommand{\bc}{\mathbf{c}}
\newcommand{\bx}{\mathbf{x}}

\newcommand{\bw}{\mathbf{w}}
\newcommand{\bW}{\mathbf{W}}
\newcommand{\bU}{\mathbf{U}}

\title{Feature-Context Consistency: Unsupervised Adversarial Detection with Drift Stability on Attributed Graphs}

\author{%
  Firas Bayram \\
  SnT, University of Luxembourg\\
  Luxembourg \\
  \texttt{firas.bayram@uni.lu} \\
  \And
  Maxime Cordy \\
  SnT, University of Luxembourg\\
  Luxembourg \\
  \texttt{maxime.cordy@uni.lu} \\
}

\begin{document}

\maketitle

\begin{abstract}
  Adversarial detection on evolving attributed graphs faces two challenges: fraud labels arrive late or not at all, and natural distributional drift erodes any trained decision boundary. Modern detectors can reach strong detection rates when labels are available and the distribution is stationary; sustaining that quality without labels and under drift is the open problem. We identify \emph{feature-context consistency} --- the alignment between a node's attributes and its aggregated neighborhood attributes --- as a single structural signal that addresses both. We formalize it as the \textbf{Context Boost} (CB) score within a Restricted Boltzmann Machine framework (CB-RBM) and prove two complementary worst-case guarantees: a strict separation between the expected CB of legitimate and adversarial nodes, with gap growing linearly in model capacity, whenever the adversary's context stays far enough from legitimate neighborhood context; and a Wasserstein-Lipschitz bound on how far the CB distribution moves under natural drift. We validate on XBlock Ethereum phishing ($\sim$153K nodes) and Reddit banned-user detection ($\sim$11K nodes), with two further datasets in the appendix spanning diverse domains. With zero labels, CB-RBM achieves AUROC~$\geq 0.96$ on four adversarial attacks and keeps FPR close to its calibrated $5\%$ target under drift --- outperforming supervised GNNs (hundreds of labels), unsupervised graph-anomaly baselines, and a standard RBM whose FPR collapses to $89$--$100\%$. What typically requires separate mechanisms --- label-free operation and drift stability --- here follows from a single structural regularity of the data. Code: \url{https://github.com/serval-uni-lu/CB-RBM}.
\end{abstract}

% ── 1. Introduction ───────────────────────────────────────────────────────────
\section{Introduction}
\label{sec:intro}

% P1 — Domain motivation: why attributed graphs:
Within networked systems, adversarial actors (e.g., scammers in cryptocurrency transaction graphs, banned users on social platforms, fake reviewers in e-commerce networks) operate at scale to bypass systems' security and business protection measures. The stakes are substantial: cryptocurrency fraud alone reached \$9.3~billion in 2024~\cite{ic3_2023,ic3_2024}, with comparable damage from coordinated review rings, account abuse, and platform manipulation~\cite{rosenberg2021adversarial}. These systems are naturally represented as \emph{attributed graphs}, where nodes carry features and edges encode relationships, and detection requires exploiting both channels together. Adversarial manipulation then consists of introducing new nodes and edges in the graph following a specific (adversarial) feature distribution.

% P2 — The two problems, given the substrate:
Yet deploying detectors on these graphs faces two compounding constraints. First, \emph{label scarcity}: adversarial entities typically operate for days or weeks before being flagged---if they are ever reported at all---leaving no reliable supervisory signal during the critical early window when the most damage is done~\cite{wu2020phishers}. Second, \emph{distributional drift}~\cite{yuan2026continual,wu2024graph}: user behavior, network topology, and feature distributions evolve continuously through organic growth and churn~\cite{auer2024technology}, inducing natural concept drift~\cite{lu2018learning}. A detector trained on a snapshot implicitly encodes the feature distribution of that moment; as legitimate entities migrate to new regions of feature space, the learned decision boundary becomes stale---legitimate activity that was once well within the normal range now falls outside it, triggering escalating false alarms~\cite{zambon2018concept}. The standard remedy is periodic retraining, but this loops back to the first problem. Retraining requires fresh labels, which arrive late and incomplete. In particular, supervised fraud GNNs~\cite{dou2020enhancing,liu2020alleviating,liu2021pick} depend on exactly these labels, and unsupervised graph detectors such as DOMINANT~\cite{ding2019deep} and CoLA~\cite{liu2021anomaly} carry no guarantee against drift-induced false alarms. Without a mechanism that is inherently stable under drift, practitioners are trapped in a cycle where the model is either stale or starved of supervision.

% P4 — The insight that solves both simultaneously:
We identify \emph{feature-context consistency}---the alignment between an entity's own attributes and the aggregated attributes of its neighbors---as a single structural signal that simultaneously addresses both problems. A legitimate entity develops feature patterns that are statistically coherent with those of its neighbors, because repeated interaction shapes both parties. An adversarial entity, by contrast, is injected without this history: its features may mimic legitimate patterns, but its neighborhood context is absent, random, or structurally inconsistent---and the adversary cannot unilaterally control the features of its neighbors. This enables \emph{label-free detection}. The same signal is also robust to drift, because natural drift moves node features and their mean-pooled context in tandem and preserves their alignment, which gives \emph{drift stability} without retraining. Both properties follow from the same structural fact: in attributed graphs with meaningful relationships, feature-context correlation is nonzero for legitimate entities and broken for early-stage adversarial injections.

% P5 — EBM/RBM as natural candidate + CB-RBM:
Operationalizing this insight requires a label-free model whose score function is analytically decomposable into feature and feature-context contributions, so the loss of coherent context for adversarial nodes is expressible in closed form. Energy-based models satisfy the label-free requirement, learning an unnormalized density from legitimate data alone~\cite{du2019implicit}. Among these, the Restricted Boltzmann Machine~\cite{hinton2002training}, a well-studied neural-network EBM, additionally provides a closed-form free energy whose per-hidden-unit decomposition supports the separation and drift bounds in \S\ref{sec:theory} --- unlike GNNs or alternative EBMs (score-based, energy-GAN). RBMs carry a track record in unsupervised fraud and anomaly detection~\cite{pumsirirat2018credit,fiore2013network,korycki2023adversarial}, yet remain unexplored for attributed-graph adversarial detection. A standard RBM, however, models features in isolation --- capturing what an entity \emph{looks like}, not whether that appearance is \emph{consistent with its surroundings}~\cite{hinton2012practical}. We extend the Gaussian-Bernoulli RBM with a gated context pathway that conditions the energy on aggregated neighborhood features, yielding the \textbf{Context-Boosted RBM (CB-RBM)}. The \textbf{Context Boost} (CB) score --- the reduction in free energy from conditioning on neighborhood context --- directly measures feature-context consistency and underlies both our detection mechanism and drift guarantees.

%  Contributions:
Our main contributions are as follows. 
\begin{itemize}[leftmargin=*,topsep=2pt,itemsep=2pt,parsep=0pt,partopsep=0pt] 
    \item \textbf{Method} (\S\ref{sec:method}). We propose \textbf{CB-RBM}, an RBM with a GLU-gated context pathway conditioning the free energy on aggregated neighborhood features, and the \textbf{Context Boost} (CB) score---the free-energy reduction from context---as a label-free detection signal.
    \item \textbf{Theoretical framework} (\S\ref{sec:theory}). Under five empirically verifiable assumptions, we prove a strict separation between the expected CB of legitimate and adversarial nodes (Theorem~\ref{thm:sep}) and a Wasserstein-Lipschitz bound on how far the CB distribution moves under drift (Theorem~\ref{thm:drift}, Corollary~\ref{cor:quantile})---coupling label-free detection and drift stability through a single feature-context correlation structure.
    \item \textbf{Detection boundary analysis} (\S\ref{sec:cb_dist},~\S\ref{sec:discussion}). The $\rho_1^{\mathrm{eff}}$ analysis reveals that CB detects \emph{absence of legitimate network context} rather than feature deviation, characterizing the detection signal and delineating its scope. 
    \item \textbf{Empirical validation} (\S\ref{sec:experiments}). On XBlock Ethereum phishing~\cite{chen2020phishing} and Reddit banned users~\cite{kumar2019predicting}, CB-RBM (0~labels) achieves AUROC~$\geq 0.96$ on all non-white-box attacks and keeps FPR close to its calibrated $5\%$ target under 20-step drift, ahead of supervised GNNs trained on hundreds of labels, unsupervised graph-anomaly baselines and a standard RBM whose FPR rises to $89$--$100\%$.

\end{itemize}

\section{Related Work}
\label{sec:related}

\textbf{Graph anomaly detection.} Supervised fraud detectors such as CARE-GNN~\cite{dou2020enhancing}, GraphConsis~\cite{liu2020alleviating}, PC-GNN~\cite{liu2021pick} and BWGNN~\cite{tang2022rethinking} learn from labeled adversaries, which arrive late in our setting. Unsupervised detectors exploit graph structure without labels. DOMINANT~\cite{ding2019deep} and AnomalyDAE~\cite{fan2020anomalydae} score attribute and structure reconstruction errors, GAD-NR~\cite{roy2024gad} reconstructs each node's neighborhood, CoLA~\cite{liu2021anomaly} contrasts each node with its local subgraph, CONAD~\cite{xu2022contrastive} adds contrastive learning over augmented anomalies, and OCGNN~\cite{wang2021one} fits a one-class boundary around normal embeddings. CB-RBM uses no reconstruction target or contrastive loss. Context enters the energy of a trained model through a gated pathway, and the score is the free-energy reduction that context produces, in closed form and decomposable over hidden units (Definition~\ref{def:cb}).

\textbf{Distribution shift and dynamic graphs.} Concept-drift methods detect a shift and then adapt or retrain~\cite{gama2014survey,lu2018learning}, and change detection extends to graph streams~\cite{zambon2018concept}. GDN~\cite{gao2023alleviating} uses labels to counter structural distribution shift between training and test nodes, and WhENDS~\cite{xu2025robust} whitens normal edge embeddings over time to track normality shift in dynamic graphs. Dynamic-graph detectors such as StrGNN~\cite{cai2021structural} and TADDY~\cite{liu2021anomaly2} model temporal evolution to score anomalous edges~\cite{ekle2024anomaly}. These approaches address shift through adaptation, robustness objectives, embedding alignment or explicit temporal modeling. In our main drift evaluation, CB-RBM keeps its trained parameters and threshold frozen, and Theorem~\ref{thm:drift} bounds how far its scores move under smooth drift.

\textbf{Generalist detectors.} ARC~\cite{liu2024arc}, UNPrompt~\cite{niu2024zero} and AnomalyGFM~\cite{qiao2025anomalygfm} train one detector on some graphs and apply it to unseen graphs with few or no target labels. UNPrompt predicts latent node attributes from neighborhoods and AnomalyGFM uses node-neighbor representation residuals. Their shift is across graphs, whereas ours is temporal drift within one deployed graph.

\textbf{Positioning.} Our contribution relative to these families is the Context Boost differential and the analysis it admits, a conditional separation of expected scores (Theorem~\ref{thm:sep}) and a Wasserstein bound on how far the scores move under drift (Theorem~\ref{thm:drift}). Section~\ref{sec:drift} and Appendix~\ref{app:baselines} compare representatives of the node-level families under our protocol.
% ── 2. Setup  ─────────────────────────────────────────────────────
\section{Problem Setup}
\label{sec:problem}

\paragraph{Graph.}
Let $\mathcal{G} = (\mathcal{V}, \mathcal{E})$ be an attributed directed graph with node features $\bx_i \in \mathbb{R}^{n_v}$ for each $i \in \mathcal{V}$. Each node represents an entity (e.g., wallet, user, reviewer) and each directed edge an interaction. Our detector uses no adversarial labels. It is trained one-class on a pool of nodes assumed legitimate, in practice accounts with no abuse report after an aging window, and its threshold is set on a held-out slice of the same pool.

\paragraph{Neighborhood context.}
The context of node $i$ aggregates the features of its neighbors $\mathcal{N}(i)$:
\begin{equation}
\bc_i = \mathbf{s} \odot \frac{1}{|\mathcal{N}(i)|} \sum_{j \in \mathcal{N}(i)} \bx_j \in \mathbb{R}^{n_v},
\label{eq:context}
\end{equation}
where $\mathbf{s} \in \{-1,+1\}^{n_v}$ is a per-dimension sign vector that corrects for anti-homophilic dimensions, ensuring that the signed context is positively correlated with node features for legitimate nodes. The construction of $\mathbf{s}$ is detailed in \S\ref{sec:method}. 
% The legitimate joint distribution $\Pl$ over $(\bx_i, \bc_i)$ pairs satisfies $\mathbb{E}[\cos(\bx_i, \bc_i)] \geq \rho_0 > 0$. 
The mean-pool aggregator is ablated in Appendix~\ref{app:ablation}.

\paragraph{Threat model.}\label{sec:threat}
An adversary aiming to evade detection injects a new node $i$ that it controls (features and edges), but not the existing nodes around it; the chosen feature distribution and neighborhood structure depend on the attack strategy. We distinguish four attack strategies of increasing sophistication:
\begin{enumerate}[leftmargin=*, itemsep=1pt]
\item \emph{Isolated Injection} --- adversarial features $\bx' \sim \Pf$ with context $\bc' \sim \Qr$, here $\mathcal{N}(\mathbf{0},\mathbf{I})$, drawn independently of $\bx'$, so the node has no legitimate neighborhood context.
\item \emph{Feature Camouflage} --- legitimate features $\bx' \sim \Pl$ with $\bc' \sim \Qr$ as above, so the node still lacks genuine context.
\item \emph{Sybil attack} --- $\bx' \sim \Pf$, and the adversary links to a fully connected clique $S$ of $50$ decoys with near-legitimate features $\bx_s = \bx_{\ell_s} + 0.05\,\boldsymbol{\epsilon}_s$ (legitimate $\ell_s$, $\boldsymbol{\epsilon}_s \sim \mathcal{N}(\mathbf{0},\mathbf{I})$), so its fabricated context $\bc'$ is the mean-pool~\eqref{eq:context} over $\mathcal{N}(i) \cup S$, where $\mathcal{N}(i)$ is the neighborhood before the attack.
\item \emph{Relation Camouflage} --- $\bx' \sim \Pf$, and an adaptive adversary links to the $10$ highest-degree legitimate nodes $H$, so its partially genuine context $\bc'$ is the mean-pool~\eqref{eq:context} over $\mathcal{N}(i) \cup H$.
% \item \emph{PGD attack} --- the adversary uses gradient access to the detector to craft features that maximize evasion under a bounded $\ell_\infty$ perturbation budget $\varepsilon$.
\end{enumerate}

These four patterns span the canonical injected-adversary threats established in graph fraud detection and adversarial graph learning~\cite{akoglu2015graph,sun2022adversarial}; attack families requiring training-time graph or label manipulation (e.g., topology poisoning~\cite{zugner2018adversarial}, backdoor attacks) lie outside this scope. We separately evaluate robustness to adversarial contamination of the training pool assumed legitimate in \S\ref{sec:robust}. White-box PGD~\cite{madry2017towards}, which needs gradient access to the detector, is a stress test in Appendix~\ref{app:pgd}.

% ── 3. Method ─────────────────────────────────────────────────────
\section{CB-RBM}\label{sec:method}

We extend the Gaussian-Bernoulli RBM~\cite{welling2004exponential}, a Gaussian-visible, Bernoulli-hidden variant suited to real-valued features, with a context pathway that conditions hidden unit activations on neighborhood features. The model's visible inputs are a node's features $\bv \in \mathbb{R}^{n_v}$ and neighborhood context $\bc \in \mathbb{R}^{n_v}$ (Eq.~\eqref{eq:context}), with binary hidden units $\mathbf{h} \in \{0,1\}^{n_h}$ as the latent layer.
    
\paragraph{Sign vector construction.}
The sign vector $\mathbf{s}$ in Eq.~\eqref{eq:context} is estimated from training data as $s_d = \mathrm{sign}(\mathrm{Pearson}(x_{i,d}, \bar{x}_{\mathcal{N}(i),d}))$, flipping anti-homophilic dimensions so that the signed context is positively correlated with node features across all dimensions. This is fitted once before training from feature statistics alone (no labels) and kept fixed. Because it needs no labels, it can be re-estimated on any new window of unlabeled nodes exactly as it was first fitted, and this refresh restores calibration after sign corruption (Appendix~\ref{app:sign}).

\paragraph{Energy function.}
The joint energy is
\begin{equation}
E(\bv, \mathbf{h}, \bc;\theta)
= \tfrac{1}{2}\lVert \bv - \mathbf{b}_v \rVert^2
- \sum_{j=1}^{n_h} h_j\big(\bw_j^\top \bv + \hat{\bc}_j + b_j\big),
\label{eq:energy}
\end{equation}
where $\hat{\bc}_j = (\bU_j^\top \bc) \cdot \sigma(\mathrm{Gate}_j^\top \bc)$ is a Gated Linear Unit~\cite{dauphin2017language} context embedding and $\theta = \{\bW, \bU, \mathrm{Gate}, \mathbf{b}_v, \mathbf{b}_h\}$. When context is absent ($\bc = 0$), the multiplicative $\bU_j^\top \bc$ factor vanishes and the model recovers a standard RBM. The free energy marginalizing over $\mathbf{h}$ is
\begin{equation}
F(\bv, \bc;\theta)
= \tfrac{1}{2}\lVert \bv - \mathbf{b}_v \rVert^2
- \sum_{j=1}^{n_h} \spl\big(\bw_j^\top \bv + \hat{\bc}_j + b_j\big).
\end{equation}

\paragraph{Context Boost score.}
\begin{definition}[Context Boost]
\label{def:cb}
The Context Boost score is
\begin{equation}
\CB(\bv, \bc;\theta)
= \sum_{j=1}^{n_h}\Big[
\spl\big(\bw_j^\top \bv + \hat{\bc}_j + b_j\big)
- \spl\big(\bw_j^\top \bv + b_j\big)
\Big].
\label{eq:cb}
\end{equation}
\end{definition}
CB measures how much neighborhood context \emph{reduces} free energy beyond features alone. The quadratic term $\tfrac{1}{2}|\bv - \mathbf{b}_v|^2$ cancels exactly in the subtraction, so CB depends only on hidden-unit activations, not on feature reconstruction cost. For a legitimate node with consistent context, the context pathway activates strongly and CB is large. For an adversarial node with absent or random context, the context contribution vanishes (algebraically when $\bc{=}0$, suppressed by the GLU gate when noisy) and CB is near zero (Theorem~\ref{thm:sep}; Figure~\ref{fig:cb_dist}).

\paragraph{Training.}
We train exclusively on legitimate node pairs $(\bx_i, \bc_i) \sim \Pl$ via Contrastive Divergence CD-1~\cite{hinton2002training}. No attack labels are used. We set the sparsity penalty $\lambda_{\mathrm{sp}} = 0$; nonzero sparsity drives mean hidden activations below 0.5 and weakens the context contribution.

\paragraph{Detection.}
The detection score for node $i$ is $s_i = -\CB(\bx_i, \bc_i)$. The threshold $\tau^*$ is set at the 95th percentile of $-\CB$ over a held-out calibration slice of the pool, setting a nominal $5\%$ calibration target. Nodes with $s_i \geq \tau^*$ are flagged as adversarial. Figure~\ref{fig:pipeline} and Algorithm~\ref{alg:cbrbm} (Appendix~\ref{app:algorithm}) summarize the full pipeline.

\paragraph{Complexity.}
Training costs $O(N \cdot n_h \cdot n_v)$ per epoch on $N$ training nodes and inference $O(n_h \cdot n_v)$ per node, as for a standard RBM, since the single-hop context is precomputed once and only closed-form free energies are evaluated~\cite{zhai2016deep}.

% ── 4. Theoretical Guarantees ─────────────────────────────────────────────────
\section{Theoretical Guarantees}
\label{sec:theory}

We establish two complementary results for the CB score. The first gives a sufficient condition under which the expected CB of legitimate nodes exceeds that of adversarial nodes. The second shows that CB remains stable under natural concept drift, explaining why a single threshold serves both detection and drift resilience. Both are worst-case certificates rather than predictions of the observed error, which Section~\ref{sec:experiments} measures directly. Full proofs are deferred to Appendix~\ref{app:proofs}. Remarks below provide intuition and justification; constants are verified in Table~\ref{tab:assumptions}.

\subsection{Assumptions}

\begin{assumption}[Signed Structural Correlation]
\label{A1}
The signed mean-pool context~\eqref{eq:context} satisfies 
\begin{equation}
\mathbb{E}_{i \sim \Pl}\big[\cos(\mathbf{x}_i, \mathbf{c}_i)\big] \ge \rho_0 > 0, \quad \text{for some } \rho_0 \in (0,1].
\end{equation}

\end{assumption}

\begin{remark}
$\mathbf{s}$ corrects anti-homophilic dimensions~\cite{zhu2020beyond} via $s_d = \mathrm{sign}(\mathrm{Pearson}(x_{i,d}, \bar{x}_{\mathcal{N}(i),d}))$, computed once on the training set. Per-dimension covariances are non-negative by construction, so $\rho_0 > 0$ whenever any dimension has nonzero neighbor covariance, which we measure directly (\S\ref{sec:experiments}).
\end{remark}

\begin{assumption}[Coherent Activation]
\label{A2}
After training, a fraction $\beta_0 \in (0,1)$ of hidden units $j$ satisfy
\begin{equation}
\mathbb{P}_{(\bv,\bc) \sim \Pl}\bigl[\alpha_j^{\mathrm{feat}} \geq 0 \;\text{and}\; \alpha_j^{\mathrm{ctx}} \geq 0\bigr] \;\geq\; p_0 > 0,
\end{equation}
where $\alpha_j^{\mathrm{feat}} = \bw_j^\top \bv + b_j$ and $\alpha_j^{\mathrm{ctx}} = \hat{\bc}_j$.
\end{assumption}
\begin{remark}
A nontrivial fraction of hidden units jointly activate on matched $(\bx,\bc)$ pairs. Under the Gaussian-projection and symmetry assumptions of the arcsine argument in Appendix~\ref{app:proofs}, $p_0 \ge \tfrac{1}{4} + \rho_0/(2\pi)$, which predicts $p_0 \ge 0.30$ on XBlock and $0.36$ on Reddit before training. After training, $\beta_0 = 0.38$ and $0.95$, that is, $97$ of $256$ and $122$ of $128$ hidden units (Table~\ref{tab:assumptions}).
\end{remark}

\begin{assumption}[Bounded Moments]
\label{A3}
$\mathbb{E}[\|\bx\|_2^2] \le M_x^2$ and $\mathbb{E}[\|\bc\|_2^2] \le M_c^2$.
\end{assumption}

\begin{assumption}[Adversarial Context Gap]
\label{A4}
The adversarial context distribution satisfies, for every hidden unit $j$,
\begin{equation}
\mathbb{E}_{\bc' \sim \Qr}\!\left[\max(\hat{\bc}'_j, 0)\right]
\;\le\;
\rho_1\,\lVert\bU_j\rVert\,M_c, \quad \text{for some } \rho_1 \in [0,\rho_0).
\end{equation}
\end{assumption}

\begin{remark}
Here $\rho_1$ bounds positive gated-context activation under adversarial context, and $\Delta\rho = \rho_0 - \rho_1 > 0$ denotes the \emph{detection gap}. Theorem~\ref{thm:sep} certifies separation when the adversarial positive-context bound lies sufficiently below the dominance margin of Assumption~\ref{A5}.
\end{remark}
   
\begin{assumption}[Coherent Dominance]
\label{A5}
The expected gain $G_{\mathcal{C}}$ of the coherent units of Assumption~\ref{A2} on their joint-positive events exceeds the expected negative contribution $N$ of all units by a margin $\gamma > 0$ independent of $n_h$, $G_{\mathcal{C}} - N \ge \gamma$ (both normalized by $n_h$ and defined in Appendix~\ref{app:sep_proof}).
\end{assumption}

\begin{remark}
Assumption~\ref{A1} captures homophily~\cite{khanam2023homophily}. Assumption~\ref{A2} formalizes the consequences of training on correlated data. Assumption~\ref{A3} imposes standard boundedness conditions. Assumption~\ref{A4} bounds the positive gated-context activation that adversarial context can produce---a realistic constraint, since attackers cannot fabricate authentic interaction history at scale. Assumption~\ref{A5} involves only the legitimate distribution and the trained model, so it is checked without adversarial data. Before training, the arcsine bound on $p_0$ checks the coherence of Assumption~\ref{A2} on which the gain $G_{\mathcal{C}}$ rests.
\end{remark}

\subsection{Separation and Stability}                                                                                         
   
Let $\Psi = \mathbb{E}[\CB \mid \Pl] > 0$ denote the expected CB under the legitimate distribution, $\Phi(\rho_1) = \mathbb{E}[\CB \mid \Pf, \Qr]$ the expected CB under adversarial features and context, and $d(\bx', \Pl) = \inf_{\bx \in \Pl}\|\bx' - \bx\|$ the feature deviation of an adversarial node from the legitimate support.  

\begin{theorem}[Separation]
\label{thm:sep}
Under Assumptions~\ref{A1}--\ref{A5}, $\Psi \ge n_h\gamma$. If, in addition, $\rho_1\bar{u}M_c \le \gamma - \varepsilon$ for some $\varepsilon > 0$ independent of $n_h$, where $\bar{u} = \tfrac{1}{n_h}\sum_j\|\bU_j\|$, then $\Psi - \Phi \ge n_h\varepsilon$, a strict separation growing linearly in $n_h$. For any threshold $\tau^* \in (\Phi, \Psi)$, Chebyshev's inequality gives
\begin{align}
\mathbb{P}_{(\bx,\bc)\sim\Pl}[\CB(\bx,\bc;\theta^*) < \tau^*] &\leq \frac{\mathrm{Var}[\CB \mid \Pl]}{(\Psi-\tau^*)^2}, \label{eq:legit}\\
\mathbb{P}_{\bx'\sim\Pf,\,\bc'\sim\Qr}[\CB(\bx',\bc';\theta^*) \ge \tau^*] &\leq \frac{\mathrm{Var}[\CB \mid \Pf,\Qr]}{(\tau^*-\Phi)^2}. \label{eq:adv}
\end{align}
Both bounds are finite since $\mathrm{Var}[\CB] \leq (\sum_j M_j)^2 < \infty$ (Lemma~\ref{lem:bounded_diff}).
\end{theorem}

\begin{proof}[Proof sketch]
\emph{Step 1: Decomposition.} Write $\CB(\bv,\bc) = \sum_j \Delta_j$ where
$\Delta_j = \spl(\alpha_j^{\mathrm{feat}} + \alpha_j^{\mathrm{ctx}}) - \spl(\alpha_j^{\mathrm{feat}})$.

\emph{Step 2: Legitimate lower bound.} On the joint-positive event (Assumption~\ref{A2}), Lemma~\ref{lem:softplus} gives $\Delta_j \ge \spl(\alpha_j^{\mathrm{ctx}}) - \log 2 \ge 0$. Summing the coherent units' gains and subtracting every unit's expected negative contribution, Assumption~\ref{A5} gives $\Psi \ge n_h(G_{\mathcal{C}} - N) \ge n_h\gamma > 0$.

\emph{Step 3: Adversarial upper bound.} Whatever the features, $\Delta_j \le \max(\hat{\bc}_j, 0)$, so Assumption~\ref{A4} gives $\Phi \le n_h\rho_1\bar{u}M_c$ and hence $\Psi - \Phi \ge n_h\varepsilon$.

\emph{Step 4: Tail bounds.} Since $\mathrm{Var}[\CB] \le (\sum_j M_j)^2$ (Lemma~\ref{lem:bounded_diff}), both variances are finite and Chebyshev yields the bounds in~\eqref{eq:legit}--\eqref{eq:adv}. See Appendix~\ref{app:sep_proof} for the full proof.
\end{proof}

\begin{corollary}[Detection Guarantee]
\label{cor:detection}
Under the conditions of Theorem~\ref{thm:sep}, setting $\tau^* = (\Psi + \Phi)/2$ gives
\[
  \mathrm{FPR} \;\leq\; \frac{4\,\mathrm{Var}[\CB \mid \Pl]}{(\Psi-\Phi)^2},
  \qquad
  1-\mathrm{TPR} \;\leq\; \frac{4\,\mathrm{Var}[\CB \mid \Pf,\Qr]}{(\Psi-\Phi)^2}.
\]
\end{corollary}

\begin{remark}
Separation is certified when the adversarial positive-context bound $\rho_1\bar{u}M_c$ lies below the legitimate dominance margin $\gamma$, and the certified margin shrinks as $\rho_1$ grows. The Chebyshev bounds are finite but can be vacuous, since $(\sum_j M_j)^2$ exceeds the measured $\mathrm{Var}[\CB]$ by $8{,}300\times$ on XBlock and $237{,}000\times$ on Reddit, so detection is measured on the full score distributions (Figure~\ref{fig:cb_dist}, Table~\ref{tab:benchmark}).
\end{remark}

\begin{theorem}[Drift Stability]
\label{thm:drift}
Let $\{(\bx_t, \bc_t) \sim \mathcal{P}_t\}_{t=0}^{T}$ undergo natural drift where:
\begin{enumerate}[label=(\roman*)]
\item Correlation structure persists: $\mathbb{E}[\cos(\bx_t, \bc_t)] \geq \rho_0$ for all $t$.
\item Drift is smooth in Wasserstein-1, $W_1(\mathcal{P}_t, \mathcal{P}_{t-1}) \leq \delta_{\mathrm{drift}}$ for all $t$, and every $\mathcal{P}_t$ is supported on $\{\|\bc\|_2 \le M_c\}$.
\end{enumerate}
Then for CB-RBM trained on $\mathcal{P}_0$,
\begin{equation}
\label{eq:driftbound}
\bigl|\mathbb{E}[\CB(\bx_t,\bc_t)] - \mathbb{E}[\CB(\bx_0,\bc_0)]\bigr|
\leq L\, t\, \delta_{\mathrm{drift}},
\end{equation}
where $L = L_v + L_c$ is the Lipschitz constant of CB on that region (Lemma~\ref{lem:lip}). Corollary~\ref{cor:horizon} gives a worst-case horizon over which the legitimate mean remains above the midpoint threshold $(\Psi+\Phi)/2$.
\end{theorem}

\begin{proof}[Proof sketch]
$\CB$ is $L$-Lipschitz (Lemma~\ref{lem:lip}); Kantorovich--Rubinstein gives $|\mathbb{E}[\CB_t]-\mathbb{E}[\CB_{t-1}]| \le L\cdot W_1(\mathcal{P}_t,\mathcal{P}_{t-1}) \le L\delta_{\mathrm{drift}}$. Telescoping over $t$ steps yields~\eqref{eq:driftbound}; full proofs, including Corollary~\ref{cor:quantile} and the drift-horizon corollary, in Appendix~\ref{app:proofs}.
\end{proof}

\begin{corollary}[Quantile Stability]
\label{cor:quantile}
Under condition~(ii) of Theorem~\ref{thm:drift}, the quantile functions $Q_t$ of the detection score $-\CB(\bx_t,\bc_t)$ satisfy $\int_0^1 |Q_t(u) - Q_0(u)|\,du \le L\,t\,\delta_{\mathrm{drift}}$.
\end{corollary}

\begin{remark}
Natural drift preserves feature-context correlation and causes CB to change gradually at rate $L\,\delta_{\mathrm{drift}}$ per step. Adversarial injection breaks this correlation and causes immediate CB suppression. This asymmetry is what enables a single threshold to simultaneously detect adversaries and tolerate drift. Corollary~\ref{cor:quantile} controls the score quantiles on average but not the 95th percentile pointwise, which Appendix~\ref{app:thrdrift} measures directly.
\end{remark}

\paragraph{General attacks.}The separation condition depends on an attack only through the positive-context bound $\rho_1$ of Assumption~\ref{A4}, since $\Phi \le n_h\rho_1\bar{u}M_c$ holds for any adversarial features. Feature deviation acts as a bounded perturbation. At fixed context, it changes each unit's contribution by at most $\|\bw_j\|_2\,d$ (Appendix~\ref{app:sep_proof}). The context-interpolation sweep in Appendix~\ref{app:taxonomy} moves the attack context from its revealing to its camouflaging end. Detection falls below $50\%$ once the camouflaging share exceeds $\alpha^*\approx 0.56$ on XBlock and $0.58$ on Reddit, and the hardest camouflage is whichever context brings $\rho_1^{\mathrm{eff}}$ closest to $\rho_0$.     

% ── 5. Experiments ─────────────────────────────────────────────────────
\section{Experiments}
\label{sec:experiments}

\subsection{Setup}
\label{sec:setup}

\paragraph{Datasets.}
We use two attributed graphs from distinct domains: \textbf{(1)~XBlock} ($153$K Ethereum wallet nodes, $0.7\%$ phishing, $38$ features, $\rho_0 = 0.30$) and \textbf{(2)~Reddit} ($11$K social-media users, $3.3\%$ banned, $73$ features, $\rho_0 = 0.68$). Context vectors are a per-dimension sign-corrected mean-pool over graph neighbors (Eq.~\eqref{eq:context}), preserving $\rho_0 > 0$ regardless of homophily direction. Full dataset details and verified Assumption~\ref{A1}--\ref{A3} constants in Appendix~\ref{app:datasets}; two further datasets --- Tolokers (crowdsourcing) and YelpChi (e-commerce review fraud) --- evaluated in Appendix~\ref{app:generalisation}.

\paragraph{Baselines.}
We benchmark CB-RBM against six baselines from the GADBench~\cite{tang2023gadbench} suite. GraphSAGE~\cite{hamilton2017inductive}, GAT~\cite{velivckovic2017graph} and GCN~\cite{kipf2016semi} represent supervised fraud GNNs, foundational variants that GADBench shows match or exceed specialized fraud-GNN refinements, and DOMINANT~\cite{ding2019deep}, CoLA~\cite{liu2021anomaly} and GAE~\cite{kipf2016variational} represent unsupervised attributed-graph anomaly detection. A standard RBM without context (Std-RBM) is the ablation. We also compare against nine detectors from the neighborhood-aware and drift-aware classes closest to our setting. These are the supervised XGBGraph~\cite{tang2023gadbench}, BWGNN~\cite{tang2022rethinking} and Graph Transformer~\cite{shi2020masked}, the shift-robust GDN~\cite{gao2023alleviating}, the generalist ARC~\cite{liu2024arc}, the reconstruction-based GAD-NR~\cite{roy2024gad} and AnomalyDAE~\cite{fan2020anomalydae}, the contrastive CONAD~\cite{xu2022contrastive} and the one-class OCGNN~\cite{wang2021one}. Two further ablations, a parameter-free signed-cosine score and a context-conditioned MLP without the RBM and GLU (Context-MLP), isolate our formulation. The nine detectors follow their published configurations, with model selection on a held-out validation split (Appendix~\ref{app:baselines}). Supervised GNNs use 932 / 166 labeled adversaries on XBlock / Reddit; all other methods use \textbf{zero labels}.

\paragraph{Metrics.}
\emph{TPR@5\%FPR}: true positive rate at a 5\%-FPR-calibrated threshold on legitimate nodes (no adversarial labels needed). \emph{Stability}: $1 - \mathrm{FPR}_{\mathrm{drift}}$, fraction of legitimate nodes correctly classified under drift. \emph{Combined Score}: $\mathrm{TPR@5\%FPR} \times \mathrm{Stability}$ per attack; we also report the mean over the four attacks as ``average Combined''.

\paragraph{CB-RBM configuration.}
$n_h{=}256$ for XBlock and $n_h{=}128$ for Reddit (capacity ablation in Appendix~\ref{app:ablation}), 100 epochs of CD-1, sparsity cost $= 0$.

% ------------------------------------------------------------
\subsection{Separation Gap and Theory Validation}
\label{sec:cb_dist}
% ------------------------------------------------------------

\begin{figure}[t]
  \centering
  \begin{subfigure}[t]{0.4\linewidth}
    \centering
    \includegraphics[width=\linewidth]{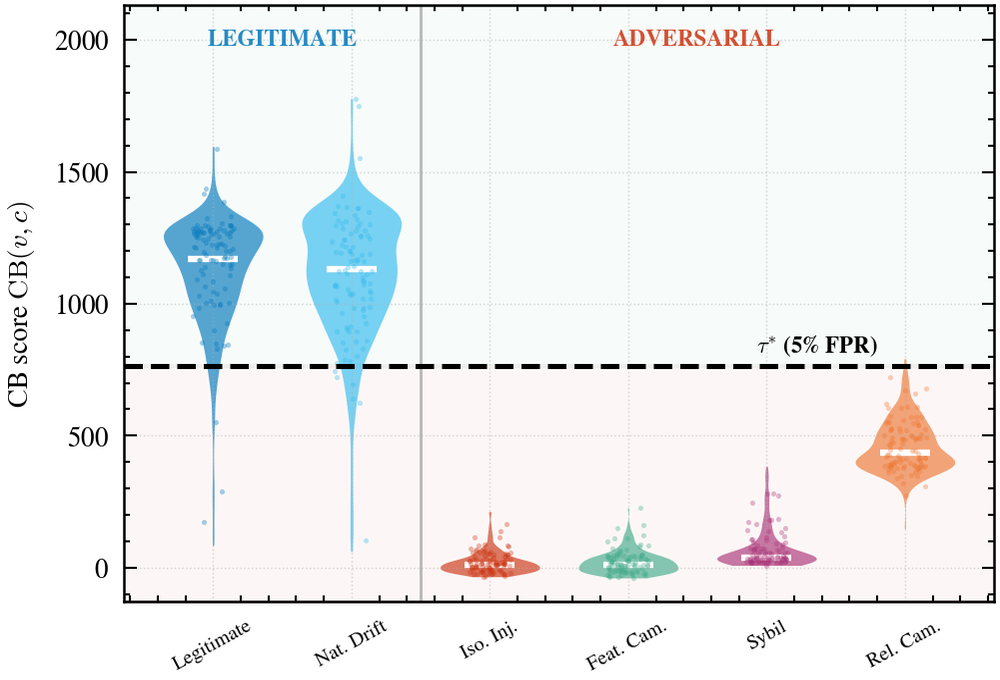}
    \caption{XBlock}
  \end{subfigure}\hspace{0.08\linewidth}
  \begin{subfigure}[t]{0.4\linewidth}
    \centering
    \includegraphics[width=\linewidth]{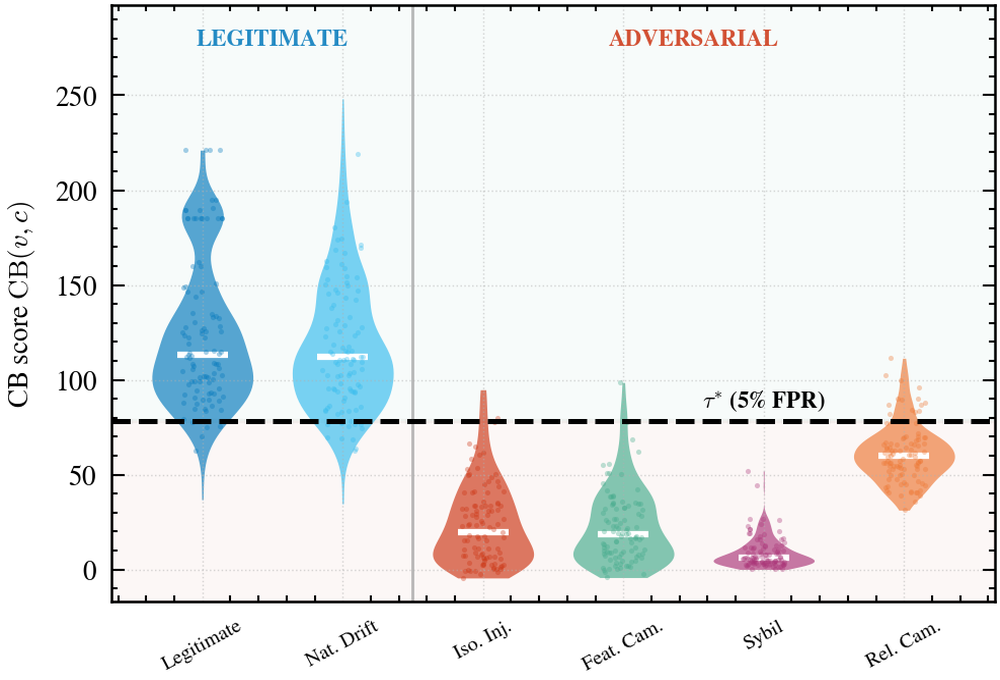}
    \caption{Reddit}
  \end{subfigure}
  \caption{CB score distributions, with calibrated threshold $\tau^{*}$ (5\% FPR).}
  \label{fig:cb_dist}
\end{figure}

\begin{figure}[t]
  \centering
  \begin{subfigure}[t]{0.4\linewidth}
    \includegraphics[width=\linewidth]{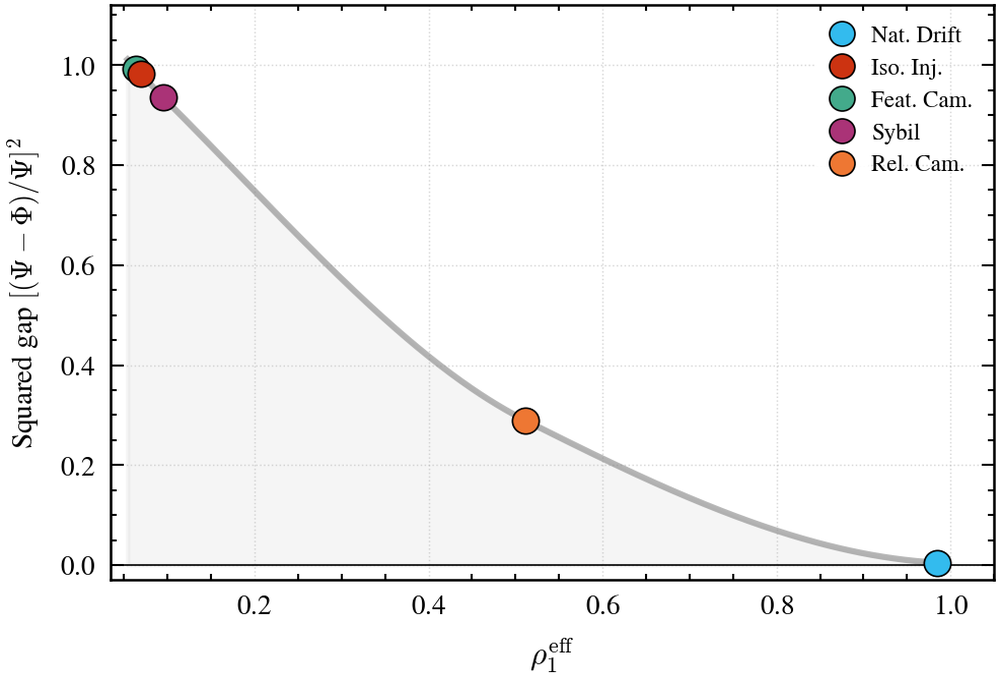}
    \caption{XBlock ($\rho_0 = 0.30$)}
  \end{subfigure}
  \hspace{0.08\linewidth}
  \begin{subfigure}[t]{0.4\linewidth}
    \includegraphics[width=\linewidth]{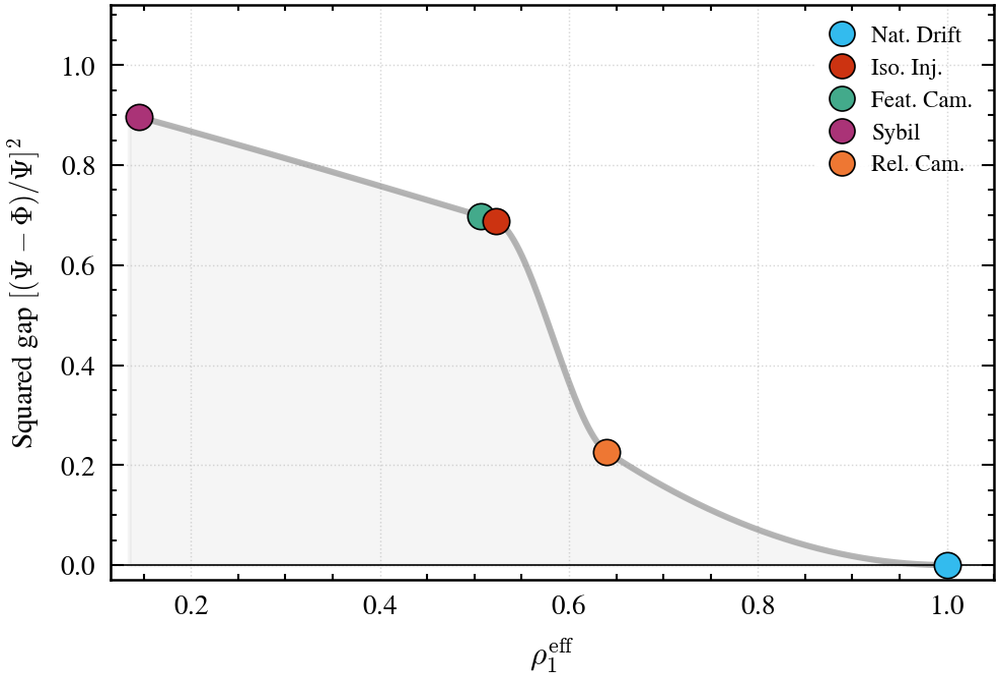}
    \caption{Reddit ($\rho_0 = 0.68$)}
  \end{subfigure}
  \caption{Squared separation gap $[(\Psi{-}\Phi)/\Psi]^2$ vs.\ $\rho_1^{\mathrm{eff}}$ (Thm.~\ref{thm:sep}).}
  \label{fig:cb_gap}
\end{figure}

A clear structural divide separates legitimate from adversarial CB score distributions on both datasets (Fig.~\ref{fig:cb_dist}): legitimate and naturally drifted nodes concentrate well above the threshold $\tau^{*}$, while all four attack types cluster predominantly below it --- regardless of whether the adversary copies features, fabricates context, or rewires edges. On each dataset, a single threshold calibrated from the legitimate distribution generalizes across all four attacks.

To understand why some attacks come closer to the threshold than others, we measure $\rho_1^{\mathrm{eff}}$, the fraction of the RBM's learned legitimate-context activation that each attack achieves (0 = no legitimate context, 1 = indistinguishable from legitimate) --- the effective adversarial correlation through the trained GLU weights, which rises when an adversary connects to legitimate hubs --- and plot it against the squared separation gap (Figure~\ref{fig:cb_gap}; the squared form reflects the Chebyshev error bound of Theorem~\ref{thm:sep}). The result traces a smooth, monotonically decreasing curve on both datasets, consistent with Theorem~\ref{thm:sep}, whose certified margin shrinks as $\rho_1$ grows. Sybil's fabricated context is rejected by the learned weights; Relation Camouflage, the only attack that acquires connections to real legitimate hubs, reaches the highest $\rho_1^{\mathrm{eff}}$ ($0.51$--$0.64$) and the steepest gap reduction. Natural Drift preserves the legitimate correlation structure ($\rho_1^{\mathrm{eff}} \approx 1$) and sits at near-zero gap --- explaining why the same threshold that catches adversaries does not trigger false alarms under drift.

% ------------------------------------------------------------
\subsection{Attack Benchmark}
\label{sec:benchmark}
% ------------------------------------------------------------

\begin{table}[t]
  \caption{Attack benchmark. Each cell: AUROC / TPR@5\%FPR (\%). Best per row in \textbf{bold}.}
  \label{tab:benchmark}
  \centering
  \setlength{\tabcolsep}{2pt}
  \renewcommand{\arraystretch}{1.05}
  \begin{tabular}{@{}llcccccccc@{}}
    \toprule
    & Attack & \multicolumn{2}{c}{Energy-based} & \multicolumn{3}{c}{Supervised} & \multicolumn{3}{c}{Unsupervised} \\
    \cmidrule(lr){3-4} \cmidrule(lr){5-7} \cmidrule(lr){8-10}
    & & CB-RBM & Std-RBM & SAGE & GAT & GCN & DOM. & CoLA & GAE \\
    \midrule
    \multirow{4}{*}{\rotatebox{90}{\small\textbf{XBlock}}}
    & Isolated Inj. & \textbf{99.9}/\textbf{100} & 86.8/51.6 & 97.6/90.1 & 91.7/75.2 & 95.4/80.8 & 96.9/78.1 & 56.4/17.2 & 98.1/89.5 \\
    & Feat.\ Camo   & \textbf{99.9}/\textbf{100} & 51.3/11.0 & 92.6/77.8 & 54.6/10.2 & 84.0/29.7 & 96.4/72.5 & 57.3/16.2 & 96.7/78.7 \\
    & Sybil         & \textbf{99.8}/\textbf{100} & 86.8/51.6 & 98.8/94.9 & 99.4/\textbf{100} & 99.6/\textbf{100} & 98.1/98.8 & 62.5/21.6 & 96.6/77.4 \\
    & Rel.\ Camo    & 97.6/96.8 & 86.8/51.6 & \textbf{98.6}/92.4 & 98.2/\textbf{97.8} & 97.8/87.8 & 97.0/86.9 & 72.1/20.2 & 97.7/88.9 \\
    \midrule
    \multirow{4}{*}{\rotatebox{90}{\small\textbf{Reddit}}}
    & Isolated Inj. & \textbf{99.4}/97.2 & 60.0/10.7 & 61.9/0.5 & 73.5/38.3 & 82.2/62.8 & 98.0/91.3 & 93.7/83.8 & 99.1/\textbf{97.3} \\
    & Feat.\ Camo   & \textbf{99.6}/\textbf{98.2} & 52.4/4.3 & 55.1/0.0 & 58.4/18.2 & 57.3/29.5 & 90.6/79.7 & 92.7/79.0 & 90.5/86.5 \\
    & Sybil         & \textbf{100}/\textbf{100} & 60.0/10.7 & 68.2/1.8 & 92.9/66.7 & 81.8/33.3 & 99.9/\textbf{100} & 85.8/43.2 & 97.4/86.8 \\
    & Rel.\ Camo    & \textbf{97.3}/\textbf{88.2} & 60.0/10.7 & 83.2/19.7 & 82.1/5.2 & 83.7/19.2 & 93.2/52.8 & 77.9/18.7 & 89.3/48.5 \\
    \bottomrule
  \end{tabular}
\end{table}

Table~\ref{tab:benchmark} compares per-attack detection. The four attacks span a spectrum of adversarial context quality, from uncorrelated (Isolated Injection, Feature Camouflage) through fabricated (Sybil) to partially genuine (Relation Camouflage). CB-RBM's AUROC tracks this spectrum cleanly on both datasets: it stays above $99\%$ on three attacks and drops to ${\sim}97\%$ on Relation Camouflage (where the adversary acquires genuine hub connections), consistent with the $\Delta\rho$-dependent separation of Theorem~\ref{thm:sep}. Across all four attacks on both datasets ($14\times$ scale, $\rho_0 = 0.30$ vs $0.68$), CB-RBM maintains AUROC~$\geq\!96\%$ and TPR@5\%FPR~$\geq\!88\%$ at zero labels. Over 3 seeds, the standard deviation of CB-RBM's AUROC is at most $0.1$ points on XBlock and $1.2$ on Reddit, and that of its TPR at most $0.3$ and $4.5$ points. Supervised GNNs, despite their full label budget, remain undertrained in this regime (best Reddit supervised result: GAT at $67\%$ TPR on Sybil). Unsupervised baselines avoid labels but their signal does not transfer reliably across datasets (CoLA: $84\%\!\to\!17\%$ TPR).    

% Adversarial detection is a mature concern in the graph-anomaly literature, but robustness alone is not sufficient for deployment; \S\ref{sec:drift} evaluates whether this quality persists under drift.
  
% ------------------------------------------------------------
\subsection{Drift Stability}
\label{sec:drift}
% ------------------------------------------------------------

\begin{figure}[t]
  \centering
  \begin{subfigure}[t]{0.45\linewidth}
    \includegraphics[width=\linewidth]{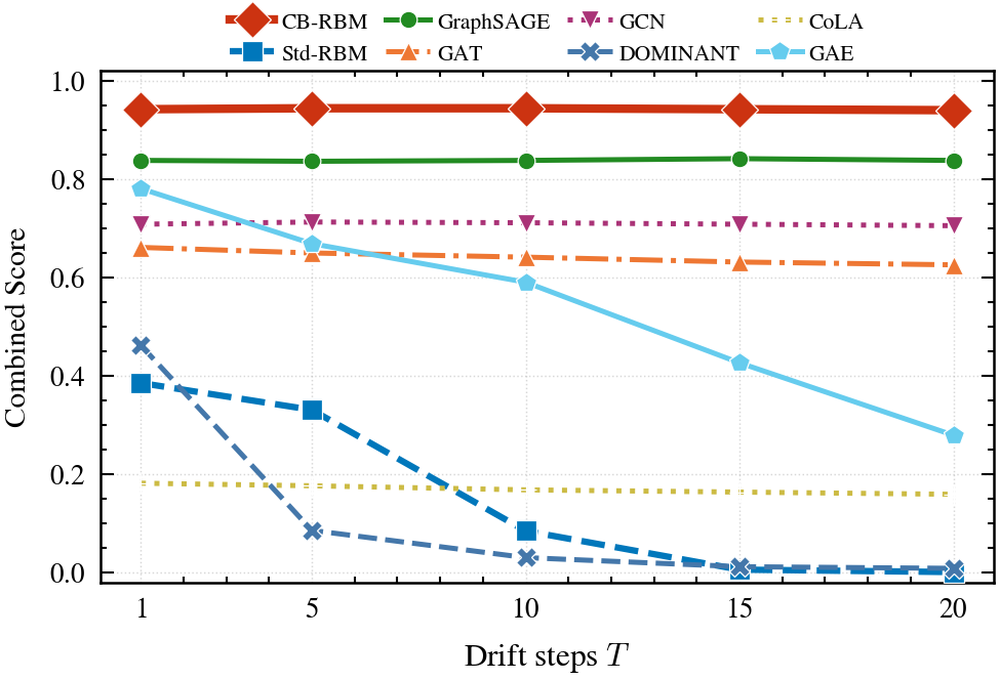}
    \caption{XBlock}
  \end{subfigure}
  \hfill
  \begin{subfigure}[t]{0.45\linewidth}
    \includegraphics[width=\linewidth]{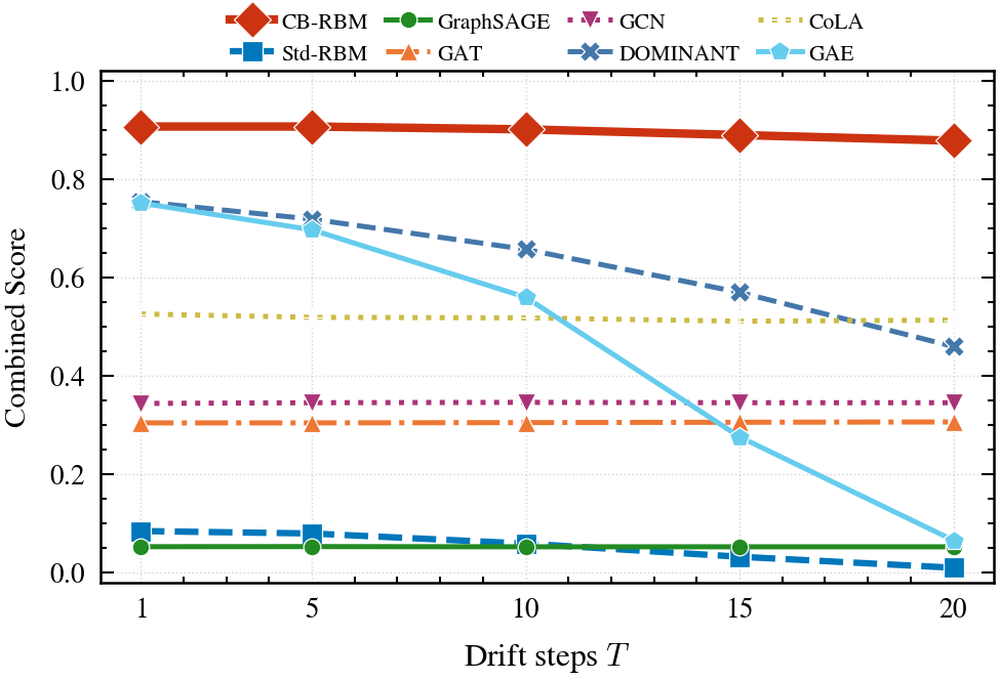}
    \caption{Reddit}
  \end{subfigure}
  \caption{Combined Score under gradual drift across 20 time steps.}
  \label{fig:drift_sim}
\end{figure}

Following standard practice in the concept drift literature~\cite{gama2014survey,lu2018learning}, we model gradual drift as an incremental random walk in feature space over $T$ steps, with neighborhood context recomputed at each step to reflect the evolving features. Figure~\ref{fig:drift_sim} reports the Combined Score (detection $\times$ stability, \S\ref{sec:setup}) as drift accumulates, using each method's mean TPR across the four attacks of Table~\ref{tab:benchmark} (FPR$_{\mathrm{drift}}$ is measured on legitimate nodes only, so we average TPR across attacks but use a single FPR curve); additional drift regimes and real temporal validation are in Appendix~\ref{app:drift_variants}.

Figure~\ref{fig:drift_sim} reveals three patterns. First, reconstruction-based detectors (GAE, DOMINANT), strong on the static benchmark (\S\ref{sec:benchmark}), collapse as their learned reconstruction targets become stale --- DOMINANT drops to near zero on XBlock, and GAE from $0.75$ to $0.07$ on Reddit. Second, supervised GNNs split by label availability. On XBlock, where labels provide a strong training signal, GNNs learn robust decision boundaries that generalize under drift, achieving Combined Scores up to $0.84$ (GraphSAGE). On Reddit, where labels are scarce, the same architectures remain undertrained (\S\ref{sec:benchmark}), and their Combined Score remains below $0.36$ throughout --- stable only because there is weak TPR to degrade. Third, CB-RBM is the only method that consistently leads on both datasets, maintaining average Combined~$\geq 0.87$ at $T{=}20$ in both the label-rich and label-scarce regimes. Theorem~\ref{thm:drift} bounds how far the scores move under smooth drift, and Figure~\ref{fig:drift_sim} measures the resulting detection and false alarms at the frozen threshold. Across five chronological XBlock windows from 2017 through 2018, with the model, sign vector and threshold frozen on the earliest window, CB-RBM keeps the legitimate-wallet FPR at or below $5\%$, reaching $3.1\%$ by late 2018, and detects $100\%$ of isolated injections built from real held-out phishing features. GAT and GraphSAGE, which score the phishing wallets in the original graph, detect $84$--$89\%$ and $81\%$ (Appendix~\ref{app:temporal_real}).

Table~\ref{tab:benchmark} and Figure~\ref{fig:drift_sim} report eight methods per attack and over the full drift trajectory. We compare the other detectors of \S\ref{sec:setup} with the same Combined Score at the last drift step, $T{=}20$ (Appendix~\ref{app:baselines}). At that step, the strongest baselines are GraphSAGE ($0.84$) and XGBGraph ($0.81$, with TPR $0.98$ but drift FPR $17.8\%$) on XBlock and GAD-NR ($0.63$) on Reddit, against $0.94$ and $0.88$ for CB-RBM. Over five seeds, the Combined Scores of XGBGraph, Graph Transformer, BWGNN and GDN vary with a standard deviation below $0.05$. GAD-NR misses the self-consistent Sybil clique on XBlock ($0\%$ TPR). GDN reaches $10.4\%$ drift FPR on XBlock and ARC $47.2\%$ on Reddit. The ablations show that each ingredient fails on its own. The RBM without context collapses under drift, the signed-cosine score reaches only $0.17$ and $0.57$ Combined, and Context-MLP reaches a static AUROC of at least $0.95$ on every Reddit attack yet raises $96\%$ false alarms under drift. The principle supplies detection, and the closed-form differential supplies the calibration that survives drift.

\subsection{Discussion}
\label{sec:discussion}

\textbf{What does CB detect?} The CB score fires on \emph{absence of legitimate neighborhood context} rather than on feature deviation itself. The GLU gate suppresses the context pathway whenever the aggregated neighborhood correlation departs from the pattern learned during training, so a node is flagged for failing to corroborate the model's learned correlation pattern, a stricter test than raw feature-context geometry. Controlled decomposition experiments (Appendix~\ref{app:decomp}) confirm this on both datasets: legitimate features paired with random or absent context are flagged as reliably as adversarial features with random context (Combined~$\geq 0.91$), while legitimate features with legitimate context consistently evade detection (Combined~$\leq 0.11$). An XBlock trace gives $\CB=1046$ for a median legitimate node, $952$ after $20$ drift steps, and $28$ for an injected adversary (Appendix~\ref{app:trace}). The same mechanism also bounds what CB-RBM can catch --- an adversary who acquires genuine network connections eventually restores the correlation signal and evades, consistent with the $\rho_1^{\mathrm{eff}}$ trajectory in Figure~\ref{fig:cb_gap}.

\textbf{Why the drift bound is conservative.} Theorem~\ref{thm:drift} bounds the worst case: it allows every feature coordinate to move adversarially within the drift budget. Natural drift does not look like this, because features and their aggregated context co-evolve through the mean-pool operator and preserve their legitimate correlation step-by-step; movement along this correlated direction lies in a low-sensitivity region of the softplus differences in Eq.~\eqref{eq:cb}. The observed CB shift at $T{=}20$ is therefore smaller than the bound evaluated at the measured value of $M_c$ by more than two orders of magnitude --- a $526\times$ slack on XBlock and a $4{,}540\times$ slack on Reddit (Table~\ref{tab:lipschitz}). At $T{=}20$, the empirical $95$th percentile of the legitimate scores shifts by $42.7$ on XBlock and $3.4$ on Reddit, while mean FPR at the frozen threshold reaches $6.1\%$ and $7.1\%$ (Appendix~\ref{app:thrdrift}).

\textbf{False-positive recovery under drift.} When adversarial context approaches legitimate context, supervised baselines with adequate labels can detect better. YelpChi, where the raw correlation gap is negative ($\Delta\rho{=}{-}0.013$), tests this regime (Appendix~\ref{app:yelpchi}). Even there, CB-RBM's drift stability pays off as a no-cost filter stacked atop those supervised baselines --- on YelpChi it eliminates $87.6$--$93.5\%$ of each baseline's drift-induced false alarms with no extra labels, training, or pipeline changes (Appendix~\ref{app:rescue}).

\textbf{Ablation.} The context pathway is what CB-RBM relies on; the surrounding architectural choices are largely fungible. Removing context (reducing the model to a free-energy RBM) collapses Combined Score by two to three orders of magnitude on both datasets, whereas varying hidden capacity within each dataset's stable range or swapping the pooling operator (mean, max, or attention) leaves detection near its ceiling (Combined~$\geq 0.87$ at the per-dataset optima; Appendix~\ref{app:ablation}).

\textbf{Robustness.}\label{sec:robust} Contaminating the training and calibration pools with up to $10\%$ fraud flips no sign dimension and keeps FPR at or below $5.0\%$, although Relation Camouflage TPR on XBlock falls to $39.0\%$ (Appendix~\ref{app:contam}). A stale sign vector raises FPR without hurting Isolated Injection detection, and re-estimating it without labels restores the $5\%$ target (Appendix~\ref{app:sign}). Drift that delays, weakens or reverses the feature-context link keeps FPR at or below $6.2\%$ on XBlock and $13.2\%$ on Reddit at $T{=}20$, and when all eight methods are recalibrated under co-evolving drift, CB-RBM keeps the highest Combined Score (Appendix~\ref{app:harddrift}). With a Sybil clique that grows from $2$ to $50$ members over $50$ steps, CB-RBM flags $100\%$ of the attackers on both datasets from step $10$, against $51.3\%$ on XBlock and $9.5\%$ on Reddit for a standard RBM (Appendix~\ref{app:sybil_grow}).

\textbf{Limitations.} CB-RBM relies on $\rho_0 > 0$, which can fail in random or adversarially curated topologies. Detection also degrades as $\Delta\rho$ collapses, allowing supervised baselines with adequate labels to perform better. The method is most effective against adversaries who lack legitimate neighborhood history; as the $\rho_1^{\mathrm{eff}}$ analysis shows, an attacker who successfully acquires genuine hub connections (Relation Camouflage) reduces the detection gap and, in the limit, becomes indistinguishable from legitimate users (Appendices~\ref{app:taxonomy} and~\ref{app:rescue}). Substantial topology change can impair calibration at the frozen threshold, and a label-free threshold refresh then restores the $5\%$ target (Appendix~\ref{app:harddrift}). Finally, white-box adversaries with gradient access to the detector can lower detection (Appendix~\ref{app:pgd}), and our drift analysis assumes smooth drift, so abrupt or adversarial distribution shifts may need mechanisms beyond the algebraic cancellation we exploit.

\section{Conclusion}
We presented CB-RBM, a label-free detector exploiting feature-context consistency: the same structural signal --- nonzero correlation between a node's features and its neighborhood context --- simultaneously enables detection (broken for injected adversaries) and keeps calibration under drift (preserved under population-level shifts). On XBlock (153K Ethereum wallets) and Reddit (11K users), CB-RBM achieves AUROC~$\geq 0.96$ on all non-white-box attacks at zero labels and maintains Combined Score~$\geq 0.87$ at $T{=}20$ under drift, while supervised GNNs underperform even with their full label set, unsupervised baselines collapse under drift, and Std-RBM loses calibration entirely. The $\rho_1^{\mathrm{eff}}$ analysis reveals a clear detection boundary: CB-RBM detects \emph{absence of legitimate network context}, not feature deviation --- explaining both its strengths and its fundamental limit against adversaries who acquire genuine connections. Two extensions follow naturally: dynamic graphs with streaming edges, where the same Lipschitz structure carries over; and porting the Context Boost principle to other generative architectures (Appendix~\ref{app:transferability}).

\begin{ack}
This work was funded in part by the Luxembourg National Research Fund (FNR), grant reference NCER22/IS/16570468/NCER-FT, through the FutureFinTech National Centre of Excellence in Research and Innovation, co-funded by the Luxembourg Ministry of Finance. For the purpose of open access, and in fulfilment of the obligations arising from the grant agreement, the authors have applied a Creative Commons Attribution 4.0 International (CC BY 4.0) license to any Author Accepted Manuscript version arising from this submission. This work was also supported by the Luxembourg AI Factory (L-AIF, grant agreement ID 1012343366). The authors declare no competing interests.
\end{ack}

\bibliographystyle{unsrtnat}
\bibliography{bib}

%%%%%%%%%%%%%%%%%%%%%%%%%%%%%%%%%%%%%%%%%%%%%%%%%%%%%%%%%%%%

\appendix

% ============================================================
%  Appendix A — Pipeline and Algorithm
% ============================================================
\clearpage
\section{Pipeline and Algorithm}
\label{app:algorithm}

Figure~\ref{fig:pipeline} summarizes CB-RBM from its inputs to the detection decision, and Algorithm~\ref{alg:cbrbm} gives the full procedure, including the label-free refresh.

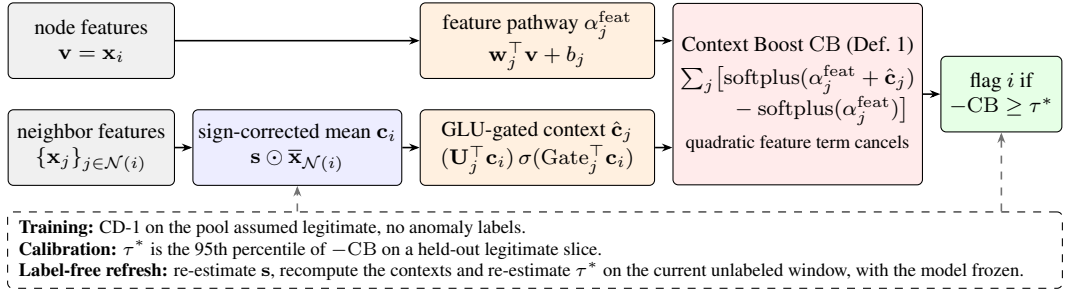
\begin{figure}[h]
\centering
\begin{tikzpicture}[
  font=\footnotesize,
  box/.style={draw, rounded corners=2pt, align=center, inner xsep=2pt, inner ysep=2pt, minimum height=1.0cm},
  arr/.style={-{Stealth[length=4pt]}, semithick},
  darr/.style={-{Stealth[length=4pt]}, semithick, dashed, black!55},
  node distance=0.23cm]
  % context row
  \node[box, fill=black!6, minimum width=2.2cm] (nb) {neighbor features\\ $\{\bx_j\}_{j\in\mathcal{N}(i)}$};
  \node[box, fill=blue!7, minimum width=2.6cm, right=of nb] (ctx) {sign-corrected mean $\bc_i$\\ $\mathbf{s}\odot\overline{\bx}_{\mathcal{N}(i)}$};
  \node[box, fill=orange!12, minimum width=3.1cm, right=of ctx] (gate) {GLU-gated context $\hat{\bc}_j$\\ $(\bU_j^\top\bc_i)\,\sigma(\mathrm{Gate}_j^\top\bc_i)$};
  % feature row
  \node[box, fill=orange!12, minimum width=3.1cm, above=0.4cm of gate] (feat) {feature pathway $\alpha_j^{\mathrm{feat}}$\\ $\bw_j^\top\bv+b_j$};
  \node[box, fill=black!6, minimum width=2.2cm] (x) at (nb |- feat) {node features\\ $\bv=\bx_i$};
  % score and decision
  \node[box, fill=red!8, minimum width=3.3cm, minimum height=2.4cm, anchor=west] (cb)
    at ($(gate.east)!0.5!(feat.east)+(0.23cm,0)$)
    {Context Boost $\CB$ (Def.~\ref{def:cb})\\[3pt]
     $\sum_j\bigl[\spl(\alpha_j^{\mathrm{feat}}+\hat{\bc}_j)$\\
     $\hphantom{\sum_j\bigl[}-\spl(\alpha_j^{\mathrm{feat}})\bigr]$\\[3pt]
     {\scriptsize quadratic feature term cancels}};
  \node[box, fill=green!10, minimum width=1.6cm, right=of cb] (dec) {flag $i$ if\\ $-\CB\geq\tau^*$};
  % data flow
  \draw[arr] (x) -- (feat);
  \draw[arr] (nb) -- (ctx);
  \draw[arr] (ctx) -- (gate);
  \draw[arr] (feat.east) -- (feat.east -| cb.west);
  \draw[arr] (gate.east) -- (gate.east -| cb.west);
  \draw[arr] (cb) -- (dec);
  % label-free maintenance
  \node[font=\scriptsize, align=left, anchor=north west, inner sep=0pt] (reftext) at ($(nb.south west)+(4pt,-0.45cm)$)
    {\textbf{Training:} CD-1 on the pool assumed legitimate, no anomaly labels.\\
     \textbf{Calibration:} $\tau^*$ is the 95th percentile of $-\CB$ on a held-out legitimate slice.\\
     \textbf{Label-free refresh:} re-estimate $\mathbf{s}$, recompute the contexts and re-estimate $\tau^*$ on the current unlabeled window, with the model frozen.};
  \node[draw, dashed, rounded corners=2pt, inner xsep=0pt, inner ysep=3pt,
        fit=(reftext.north -| nb.west) (reftext.south -| dec.east)] (ref) {};
  \draw[darr] (ref.north -| ctx.south) -- (ctx.south);
  \draw[darr] (ref.north -| dec.south) -- (dec.south);
\end{tikzpicture}
\caption{CB-RBM pipeline. The sign-corrected neighbor mean $\bc_i$ enters only through the GLU-gated pathway, and the Context Boost is the free-energy reduction due to context, in which the quadratic feature term cancels. The sign vector $\mathbf{s}$ and the threshold $\tau^*$ are label-free statistics that can be refreshed on a current unlabeled window without retraining (dashed arrows, Algorithm~\ref{alg:cbrbm}).}
\label{fig:pipeline}
\end{figure}

\begin{algorithm}[h]
\caption{CB-RBM: training, calibration, scoring, and label-free refresh}
\label{alg:cbrbm}
\small
\textbf{Input:} graph $\mathcal{G}$, pool $\mathcal{V}_{\mathrm{L}}$ assumed legitimate (no anomaly labels), calibration slice $\mathcal{V}_{\mathrm{cal}} \subset \mathcal{V}_{\mathrm{L}}$, capacity $n_h$, target FPR $5\%$.
\begin{enumerate}[leftmargin=*, itemsep=1pt, topsep=2pt, label=\arabic*:]
\item \emph{Sign vector.} $s_d \leftarrow \mathrm{sign}\bigl(\mathrm{Pearson}_{i \in \mathcal{V}_{\mathrm{L}}}(x_{i,d}, \bar{x}_{\mathcal{N}(i),d})\bigr)$ for $d = 1,\dots,n_v$ \hfill (features only)
\item \emph{Context.} $\bc_i \leftarrow \mathbf{s} \odot \tfrac{1}{|\mathcal{N}(i)|}\sum_{j \in \mathcal{N}(i)} \bx_j$ for all $i$ \hfill Eq.~\eqref{eq:context}
\item \emph{Training.} Fit $\theta = \{\bW, \bU, \mathrm{Gate}, \mathbf{b}_v, \mathbf{b}_h\}$ by CD-1 on $\{(\bx_i, \bc_i)\}_{i \in \mathcal{V}_{\mathrm{L}} \setminus \mathcal{V}_{\mathrm{cal}}}$ with $\lambda_{\mathrm{sp}} = 0$
\item \emph{Calibration.} $\tau^* \leftarrow$ 95th percentile of $\{-\CB(\bx_i, \bc_i)\}_{i \in \mathcal{V}_{\mathrm{cal}}}$ \hfill Eq.~\eqref{eq:cb}
\item \emph{Scoring.} At time $t$, recompute $\bc_i$ from the current neighbors of node $i$ and flag $i$ if $-\CB(\bx_i, \bc_i) \ge \tau^*$, with $\theta$ frozen
\item \emph{Label-free refresh} (optional, e.g., on a schedule or after substantial topology change). Re-estimate $\mathbf{s}$ as in line~1, recompute the contexts as in line~2 and re-estimate $\tau^*$ as in line~4 on the current unlabeled window, keeping $\theta$ frozen. No labels and no retraining.
\end{enumerate}
\end{algorithm}

% ============================================================
%  Appendix A — Full Proofs
% ============================================================
\section{Full Proofs}
\label{app:proofs}

We restate all results and provide complete proofs. Recall the CB score
\begin{equation}
  \CB(\bv, \bc;\,\theta)
  \;=\;
  \sum_{j=1}^{n_h}
  \underbrace{\Bigl[
    \spl\!\bigl(\alpha_j^{\mathrm{feat}} + \alpha_j^{\mathrm{ctx}}\bigr)
    -
    \spl\!\bigl(\alpha_j^{\mathrm{feat}}\bigr)
  \Bigr]}_{\displaystyle \Delta_j(\bv,\bc)},
  \label{eq:cb_def_app}
\end{equation}
where $\spl(x)=\log(1+e^x)$,
$\alpha_j^{\mathrm{feat}} = \bw_j^\top\bv + b_j$, and $\alpha_j^{\mathrm{ctx}} = \hat{\bc}_j$.

% ------------------------------------------------------------
\subsection{Supporting Lemmas}
% ------------------------------------------------------------

\begin{lemma}[Softplus monotone increment]
\label{lem:softplus}
For all $a, b \ge 0$,
\[
  \spl(a + b) - \spl(a) \;\ge\; \spl(b) - \log 2.
\]
\end{lemma}

\begin{proof}
Write
\[
  \spl(a+b) - \spl(a)
  = \log\!\frac{1+e^{a+b}}{1+e^a}
  = \log\!\bigl((1-\sigma(a)) + \sigma(a)\,e^b\bigr),
\]
where $\sigma(a)=e^a/(1+e^a)$ is the sigmoid.
For $a\ge 0$, $\sigma(a)\ge\tfrac{1}{2}$, and for $b\ge 0$, $e^b\ge 1$, so
\[
  (1-\sigma(a)) + \sigma(a)\,e^b
  = 1 + \sigma(a)(e^b - 1)
  \;\ge\; 1 + \tfrac{1}{2}(e^b - 1)
  = \frac{1+e^b}{2}.
\]
Therefore $\spl(a+b)-\spl(a) \ge \log\frac{1+e^b}{2} = \spl(b) - \log 2$.
\end{proof}

\begin{lemma}[Range of $\Delta_j$]
\label{lem:bounded_diff}
For each hidden unit $j$ and all $(\bv,\bc)$, $|\Delta_j(\bv,\bc)| \le \|\bU_j\|\,\|\bc\|$. Hence, under Assumption~\ref{A3}, $\mathbb{E}[\Delta_j^2] \le M_j^2$ with $M_j = \|\bU_j\|\,M_c$, and $\mathrm{Var}[\CB] \le (\sum_j M_j)^2$.
\end{lemma}

\begin{proof}
Since $\spl$ is monotonically increasing with $\spl'=\sigma\in(0,1)$, the mean-value theorem gives
\[
\begin{aligned}
  |\Delta_j|
  &= \bigl|\spl\bigl(\alpha_j^{\mathrm{feat}}+\alpha_j^{\mathrm{ctx}}\bigr)
      -\spl\bigl(\alpha_j^{\mathrm{feat}}\bigr)\bigr| \\
  &\le |\alpha_j^{\mathrm{ctx}}|
   = |\hat{\bc}_j| \\
  &\le \|\bU_j\|\,\|\bc\|.
\end{aligned}
\]
Squaring and taking expectations under Assumption~\ref{A3} gives $\mathbb{E}[\Delta_j^2] \le \|\bU_j\|^2 M_c^2 = M_j^2$, and Minkowski's inequality gives $\mathbb{E}[\CB^2]^{1/2} \le \sum_j \mathbb{E}[\Delta_j^2]^{1/2} \le \sum_j M_j$, so $\mathrm{Var}[\CB] \le (\sum_j M_j)^2$.
The sign of $\Delta_j$ equals the sign of $\alpha_j^{\mathrm{ctx}}$ (by monotonicity of $\spl$), so both positive and negative values occur.
\end{proof}

\begin{lemma}[Lipschitz continuity of CB]
\label{lem:lip}
For any $(\bv_1,\bc_1)$ and $(\bv_2,\bc_2)$ with $\|\bc_1\|_2, \|\bc_2\|_2 \le M_c$,
\[
  \bigl|\CB(\bv_1,\bc_1) - \CB(\bv_2,\bc_2)\bigr|
  \;\le\;
  L_v\,\|\bv_1-\bv_2\| + L_c\,\|\bc_1 - \bc_2\|,
\]
where $L_v = \sum_j\|\bw_j\|$ and $L_c = \sum_j\|\bU_j\|\bigl(1+\tfrac{1}{4}\|\mathrm{Gate}_j\|\,M_c\bigr)$.
In particular, $\CB$ is $L$-Lipschitz with $L=L_v+L_c$ on $\{(\bv,\bc) \in \mathbb{R}^{2n_v} : \|\bc\|_2 \le M_c\}$.
\end{lemma}

\begin{proof}
For any $j$, $\nabla_{\bv}\Delta_j
= \bw_j\bigl[\sigma(\alpha_j^{\mathrm{feat}}+\alpha_j^{\mathrm{ctx}})
             -\sigma(\alpha_j^{\mathrm{feat}})\bigr]$.
Since each sigmoid value lies in $(0,1)$, the difference in brackets lies in $(-1,1)$, so
$\|\nabla_{\bv}\Delta_j\| \le \|\bw_j\|$.
By the triangle inequality,
$\|\nabla_{\bv}\CB\| = \|\sum_j\nabla_{\bv}\Delta_j\|
 \le \sum_j\|\bw_j\| = L_v$.

For the context channel,
the GLU gate gives $\alpha_j^{\mathrm{ctx}} = (\bU_j^\top\bc)\,\sigma(\mathrm{Gate}_j^\top\bc)$, so
\[
  \nabla_{\bc}\alpha_j^{\mathrm{ctx}}
  = \bU_j\,\sigma(\mathrm{Gate}_j^\top\bc)
  + (\bU_j^\top\bc)\,\sigma'(\mathrm{Gate}_j^\top\bc)\,\mathrm{Gate}_j.
\]
Since $\sigma\le 1$, $\sigma'\le\tfrac{1}{4}$, and $|\bU_j^\top\bc|\le\|\bU_j\|\,M_c$ for $\|\bc\|_2 \le M_c$,
$\|\nabla_{\bc}\alpha_j^{\mathrm{ctx}}\| \le \|\bU_j\| + \tfrac{1}{4}\|\bU_j\|\,M_c\,\|\mathrm{Gate}_j\|$.
Multiplying by $\sigma(\alpha_j^{\mathrm{feat}}+\alpha_j^{\mathrm{ctx}})\le 1$ (from $\spl'=\sigma$) and summing:
$\|\nabla_{\bc}\CB\|\le L_c$.
Combining via the mean-value theorem on each channel yields the result.
\end{proof}

% ------------------------------------------------------------
\subsection{Proof of Theorem~\ref{thm:sep} (Separation)}
\label{app:sep_proof}
% ------------------------------------------------------------

\begin{proof}
The proof proceeds in four steps: decomposition, legitimate lower bound, adversarial upper bound, and tail bounds.

\medskip
\textbf{Step 1: Decomposition.}
Write $\CB(\bv,\bc)=\sum_{j=1}^{n_h}\Delta_j$ as in~\eqref{eq:cb_def_app}. We will analyze the per-unit contributions $\Delta_j$ separately for the coherent units and their complement.

\medskip
\textbf{Step 2: Lower bound for legitimate nodes ($\Psi>0$).}
Let $\mathcal{C}\subset\{1,\dots,n_h\}$ denote the set of \emph{coherent} hidden units (those satisfying Assumption~\ref{A2}), $|\mathcal{C}| \ge \beta_0 n_h$, and for each $j \in \mathcal{C}$ let $E_j = \{\alpha_j^{\mathrm{feat}} \ge 0 \;\wedge\; \alpha_j^{\mathrm{ctx}} \ge 0\}$ be its joint-positive event. On $E_j$, Lemma~\ref{lem:softplus} (with $a=\alpha_j^{\mathrm{feat}}\ge 0$, $b=\alpha_j^{\mathrm{ctx}}\ge 0$) gives
\begin{equation}
  \Delta_j
  \;\ge\; \spl(\alpha_j^{\mathrm{ctx}}) - \log 2
  \;\ge\; \spl(0) - \log 2 \;=\; 0.
  \label{eq:delta_lb}
\end{equation}
Write $\Delta_j = \Delta_j^{+} - \Delta_j^{-}$ with $\Delta_j^{\pm} = \max(\pm\Delta_j, 0)$, and define the conditional gains $g_j = \mathbb{E}[\spl(\alpha_j^{\mathrm{ctx}}) - \log 2 \mid E_j]$ for $j\in\mathcal{C}$, the normalized expected gain $G_{\mathcal{C}} = \tfrac{1}{n_h}\sum_{j\in\mathcal{C}} \mathbb{P}[E_j]\, g_j$ and the normalized expected negative contribution $N = \tfrac{1}{n_h}\sum_{j=1}^{n_h} \mathbb{E}[\Delta_j^{-} \mid \Pl]$. By~\eqref{eq:delta_lb}, $\Delta_j^{+} \ge \Delta_j\,\mathbf{1}_{E_j} \ge (\spl(\alpha_j^{\mathrm{ctx}}) - \log 2)\,\mathbf{1}_{E_j}$ for $j\in\mathcal{C}$, and $\Delta_j^{+} \ge 0$ for every unit, so
\begin{equation}
  \Psi
  \;=\; \sum_{j=1}^{n_h} \mathbb{E}[\Delta_j^{+} - \Delta_j^{-} \mid \Pl]
  \;\ge\; \sum_{j\in\mathcal{C}} \mathbb{P}[E_j]\, g_j \;-\; \sum_{j=1}^{n_h} \mathbb{E}[\Delta_j^{-}\mid\Pl]
  \;=\; n_h\bigl(G_{\mathcal{C}} - N\bigr).
  \label{eq:cb_split}
\end{equation}
Assumption~\ref{A5} gives $G_{\mathcal{C}} - N \ge \gamma > 0$ with $\gamma$ independent of $n_h$, hence
\begin{equation}
  \Psi \;=\; \mathbb{E}[\CB\mid\Pl] \;\ge\; n_h\,\gamma \;=\; \Omega(n_h) \;>\; 0.
  \label{eq:psi_lb}
\end{equation}

\medskip
\textbf{Step 3: Upper bound for adversarial nodes.}
For every unit and every input, $\Delta_j \le \max(\hat{\bc}_j, 0)$. If $\hat{\bc}_j \le 0$, monotonicity of $\spl$ gives $\Delta_j \le 0$. If $\hat{\bc}_j > 0$, the mean-value theorem with $\spl' = \sigma < 1$ gives $\Delta_j \le \hat{\bc}_j$. The bound holds whatever the features, so Assumption~\ref{A4} gives
\begin{equation}
  \Phi
  \;\le\; \sum_{j=1}^{n_h} \mathbb{E}_{\bc'\sim\Qr}\bigl[\max(\hat{\bc}'_j, 0)\bigr]
  \;\le\; \rho_1 M_c \sum_{j=1}^{n_h}\|\bU_j\|
  \;=\; n_h\,\rho_1\,\bar{u}\,M_c.
  \label{eq:phi_ub}
\end{equation}
With~\eqref{eq:psi_lb} and the condition $\rho_1\bar{u}M_c \le \gamma - \varepsilon$ of Theorem~\ref{thm:sep},
\[
  \Psi - \Phi \;\ge\; n_h\bigl(\gamma - \rho_1\bar{u}M_c\bigr) \;\ge\; n_h\,\varepsilon \;=\; \Omega(n_h).
\]
The certified margin $n_h(\gamma - \rho_1\bar{u}M_c)$ decreases linearly as $\rho_1$ grows.

\emph{Degenerate context paths.} A unit with $\bU_j = \mathbf{0}$ has $\hat{\bc}_j \equiv 0$, hence $\Delta_j \equiv 0$, and contributes nothing to $G_{\mathcal{C}}$, $N$ or~\eqref{eq:phi_ub}. Units with small $\|\bU_j\|$ contribute proportionally little to the adversarial bound.

\emph{Feature channel.} Feature deviation enters only as a bounded perturbation. At fixed context, the gradient bound $\|\nabla_{\bv}\Delta_j\| \le \|\bw_j\|$ from the proof of Lemma~\ref{lem:lip} gives $|\Delta_j(\bv',\bc) - \Delta_j(\bv,\bc)| \le \|\bw_j\|_2\,\|\bv' - \bv\|$, so replacing legitimate features by adversarial features at distance $d$ moves each unit by at most $\|\bw_j\|_2\,d$ and the score by at most $L_v\,d$. The bound~\eqref{eq:phi_ub} does not rely on this channel.

\medskip
\textbf{Step 4: Tail bounds.}
Steps~2--3 establish $\Psi-\Phi>0$.
For $\tau^* \in (\Phi,\Psi)$, Chebyshev's inequality gives
\begin{equation}
\label{eq:chebyshev_legit}
  \mathbb{P}[\CB < \tau^* \mid \Pl]
  \;\le\; \frac{\mathrm{Var}[\CB\mid\Pl]}{(\Psi-\tau^*)^2}.
\end{equation}
By Lemma~\ref{lem:bounded_diff}, $\mathrm{Var}[\CB\mid\Pl] \le (\sum_j M_j)^2 < \infty$.
Setting $\tau^*=(\Psi+\Phi)/2$:
\begin{equation}
\label{eq:fpr_bound}
  \mathrm{FPR}
  \;\le\; \frac{4\,\mathrm{Var}[\CB\mid\Pl]}{(\Psi-\Phi)^2}.
\end{equation}
The adversarial tail follows identically:
$1-\mathrm{TPR} \le 4\,\mathrm{Var}[\CB\mid\Pf,\Qr]\,/\,(\Psi-\Phi)^2$.
Since $\Psi-\Phi \ge n_h\varepsilon$ (Steps~2--3) and the variances are bounded above by $(\sum_j M_j)^2 = O(n_h^2)$, the bounds are finite for any $\tau^*\in(\Phi,\Psi)$ but need not fall below one.
They certify finiteness and do not predict the observed error rates.
\end{proof}

\begin{remark}[Rate under uncorrelated units]
\label{rem:exp_tails}
If the per-unit contributions $\{\Delta_j\}$ are pairwise uncorrelated under $\Pl$ (which holds approximately when the weight vectors span diverse directions), then $\mathrm{Var}[\CB] \le \sum_j \mathrm{Var}[\Delta_j] \le n_h\overline{M}^2$ with $\overline{M} = \max_j M_j$, and the right-hand side of~\eqref{eq:fpr_bound} becomes $O(1/n_h)$.
This refinement requires assumptions beyond Assumptions~\ref{A1}--\ref{A5}, and no bound in this paper relies on it.

\end{remark}

% ------------------------------------------------------------
\subsection{Proof of Corollary~\ref{cor:detection} (Detection Guarantee)}
% ------------------------------------------------------------

\begin{proof}
Set $\tau^*=(\Psi+\Phi)/2$, so $\Psi-\tau^*=\tau^*-\Phi=(\Psi-\Phi)/2>0$.
Applying~\eqref{eq:chebyshev_legit}:
\begin{align*}
  \mathrm{FPR}
  &\;\le\; \frac{4\,\mathrm{Var}[\CB\mid\Pl]}{(\Psi-\Phi)^2},\\[4pt]
  1-\mathrm{TPR}
  &\;\le\; \frac{4\,\mathrm{Var}[\CB\mid\Pf,\Qr]}{(\Psi-\Phi)^2}.
\end{align*}
Both variances are finite by Lemma~\ref{lem:bounded_diff}, and $\Psi-\Phi>0$ by Theorem~\ref{thm:sep}, so both bounds are finite.
\end{proof}

% ------------------------------------------------------------
\subsection{Proof of Theorem~\ref{thm:drift} (Drift Stability)}
% ------------------------------------------------------------

\begin{proof}
\textbf{Step 1: Per-step Wasserstein--Lipschitz bound.}
By Lemma~\ref{lem:lip}, $\CB$ is $L$-Lipschitz with $L = L_v + L_c$ on the region $\{\|\bc\|_2 \le M_c\}$ that carries every $\mathcal{P}_t$ (condition~(ii)). Kantorovich--Rubinstein duality gives
\[
  \bigl|\mathbb{E}_{\mathcal{P}_t}[\CB] - \mathbb{E}_{\mathcal{P}_{t-1}}[\CB]\bigr|
  \;\le\; L \cdot W_1(\mathcal{P}_t, \mathcal{P}_{t-1})
  \;\le\; L\,\delta_{\mathrm{drift}},
\]
where the first inequality is the dual characterization of $W_1$ for $L$-Lipschitz observables and the second is condition~(ii).

\textbf{Step 2: Telescoping.}
By the triangle inequality,
\begin{align*}
  \bigl|\mathbb{E}[\CB_t] - \mathbb{E}[\CB_0]\bigr|
  &\;\le\; \sum_{s=1}^{t}\bigl|\mathbb{E}[\CB_s]-\mathbb{E}[\CB_{s-1}]\bigr|
  \;\le\; L\,t\,\delta_{\mathrm{drift}},
\end{align*}
which establishes~\eqref{eq:driftbound}.

\textbf{Step 3: Drift horizon for the legitimate mean.}
Combining Step~2 with the legitimate baseline $\mathbb{E}[\CB \mid \mathcal{P}_0] = \Psi$ gives
\[
  \mathbb{E}[\CB \mid \mathcal{P}_t] \;\ge\; \Psi - L\,t\,\delta_{\mathrm{drift}}.
\]
Condition~(i) preserves the assumed feature-context correlation, and $\Psi = \mathbb{E}_{\mathcal{P}_0}[\CB]$ is the initial legitimate baseline. With $\tau^* = (\Psi+\Phi)/2$, the threshold remains strictly below $\mathbb{E}[\CB \mid \mathcal{P}_t]$ whenever
\[
  L\,t\,\delta_{\mathrm{drift}} \;<\; \tfrac{1}{2}(\Psi-\Phi)
  \quad\Longleftrightarrow\quad
  t \;<\; \frac{\Psi-\Phi}{2L\,\delta_{\mathrm{drift}}},
\]
the drift horizon~\eqref{eq:drifthorizon}. Within this horizon, the variance argument of Corollary~\ref{cor:detection} applies at time $t$ with gap $\mathbb{E}[\CB \mid \mathcal{P}_t] - \tau^*$ (variances remain bounded by Lemma~\ref{lem:bounded_diff}).
\end{proof}

\begin{corollary}[Drift Horizon]
\label{cor:horizon}
Under the conditions of Theorem~\ref{thm:drift}, with $\tau^* = (\Psi+\Phi)/2$ as in Corollary~\ref{cor:detection}, the legitimate mean stays above the threshold, $\mathbb{E}[\CB \mid \mathcal{P}_t] > \tau^*$, for any $t$ in the worst-case \emph{drift horizon}
\begin{equation}
\label{eq:drifthorizon}
t \;<\; \frac{\Psi - \Phi}{2L\,\delta_{\mathrm{drift}}}.
\end{equation}
\end{corollary}

\begin{proof}
Step~3 of the proof of Theorem~\ref{thm:drift} establishes $\mathbb{E}[\CB \mid \mathcal{P}_t] \ge \Psi - L t \delta_{\mathrm{drift}}$. This exceeds $\tau^* = (\Psi+\Phi)/2$ exactly when $L t \delta_{\mathrm{drift}} < (\Psi-\Phi)/2$, i.e., $t < (\Psi-\Phi)/(2L\delta_{\mathrm{drift}})$. The variances remain bounded by Lemma~\ref{lem:bounded_diff}, so Chebyshev's inequality still applies at time $t$, with gap $\mathbb{E}[\CB \mid \mathcal{P}_t] - \tau^*$ in place of $(\Psi-\Phi)/2$.
\end{proof}

% ------------------------------------------------------------
\subsection{Proof of Corollary~\ref{cor:quantile} (Quantile Stability)}
% ------------------------------------------------------------

\begin{proof}
Let $\mu_t$ be the law of the detection score $-\CB(\bx_t,\bc_t)$ under $\mathcal{P}_t$. Any coupling of $\mathcal{P}_t$ and $\mathcal{P}_0$ maps through $-\CB$ to a coupling of $\mu_t$ and $\mu_0$, and $-\CB$ is $L$-Lipschitz (Lemma~\ref{lem:lip}), so $W_1(\mu_t,\mu_0) \le L\,W_1(\mathcal{P}_t,\mathcal{P}_0)$. The triangle inequality for $W_1$ and condition~(ii) give $W_1(\mathcal{P}_t,\mathcal{P}_0) \le t\,\delta_{\mathrm{drift}}$. On the real line the monotone quantile coupling is optimal, so $W_1(\mu_t,\mu_0) = \int_0^1 |Q_t(u) - Q_0(u)|\,du$, which proves the claim. The bound controls the average displacement of the quantiles, not any single quantile such as the 95th percentile that sets $\tau^*$.
\end{proof}

\begin{remark}[Tightness]
The Wasserstein bound is a worst-case upper bound, loose for natural drift, since real drift moves $\bx$ and $\bc$ jointly along directions where the softplus difference~\eqref{eq:cb} is nearly flat. We discuss this slack quantitatively in~\S\ref{sec:discussion}.
\end{remark}

\begin{remark}[Coherent Activation justification]
\label{rem:a2_justification}
Assumption~\ref{A2} (joint-positive activation for a fraction $\beta_0$ of hidden units with probability $\ge p_0$) follows qualitatively from contrastive-divergence training on the correlated $(\bx,\bc)$ pairs of Assumption~\ref{A1}: matched pairs raise both $\alpha_j^{\mathrm{feat}} = \bw_j^\top\bx + b_j$ and $\alpha_j^{\mathrm{ctx}} = \hat{\bc}_j$ on the same training events, since the model is rewarded for assigning low free energy when both channels fire coherently. In the Gaussian-projection regime, the arcsine law~\cite{goemans1995improved} relates Assumption~\ref{A1}'s correlation $\rho_0$ to a sign-correlation lower bound $\tfrac{2}{\pi}\arcsin(\rho_0) \ge \tfrac{2}{\pi}\rho_0$, which under symmetric pre-activation marginals translates to $p_0 \ge \tfrac{1}{4} + \tfrac{\rho_0}{2\pi}$. The constant $\beta_0$ is measured after model convergence and $p_0$ is bounded by the arcsine argument above.
\end{remark}

% ============================================================
%  Appendix B — Datasets
% ============================================================
\section{Datasets}
\label{app:datasets}

\textbf{XBlock}~\cite{chen2020phishing} is an Ethereum transaction graph with $153{,}477$ wallet nodes (1{,}127 labeled phishing, 0.7\%) and $715{,}573$ directed transaction edges. Node features are $38$ per-wallet transaction statistics (counts, incoming/outgoing volumes, gas usage, temporal spread). The label-graph correlation is anti-homophilic. The sign-corrected feature-context correlation is $\rho_0 = 0.30$ (Eq.~\eqref{eq:context}).
\textbf{Reddit}~\cite{kumar2019predicting} is a social-media graph with $10{,}984$ user nodes (366 banned users, 3.3\%) and $168{,}016$ edges between users who co-participate in subreddits. Features are $73$-dimensional, combining $64$ LIWC linguistic categories computed over each user's posts with $9$ graph-structural signals (degree, clustering, centrality variants). Labels correlate homophilically with the graph. The sign-corrected feature-context correlation is $\rho_0 = 0.68$.
Features are globally normalized. A $10\%$ held-out split of the pool assumed legitimate, which the datasets' benign labels define in our experiments, is used for threshold calibration and drift-induced FPR measurement, and a fixed attack-generation reserve of $\sim$200 labeled adversaries is held out per dataset, so the supervised baselines see the remaining $932$ / $166$ adversary labels; all other nodes are used for training. Without drift, the Reddit threshold is stable down to small calibration samples, with FPR $5.9 \pm 0.6\%$ at $530$ calibration nodes and $6.0 \pm 0.2\%$ at $1{,}061$ over three sampling seeds, and Isolated Injection TPR $93.7 \pm 0.2\%$ in that run.

\paragraph{Verified theorem assumptions.} The constants entering Theorems~\ref{thm:sep}--\ref{thm:drift} are measured on training data: $\rho_0 = \tfrac{1}{N}\sum_i \cos(\bx_i, \bc_i)$ on the pool assumed legitimate before training; $\beta_0$, $M_x$, $M_c$ after model convergence (Table~\ref{tab:assumptions}). $M_x$ and $M_c$ are root-mean-square (RMS) norms over that pool, and the drift budgets of Table~\ref{tab:lipschitz} evaluate $L$ at this value of $M_c$ rather than at a verified support radius. Assumption~\ref{A5} holds for the trained models on both datasets.

\begin{table}[t]
\centering
\small
\caption{Verified Assumption~\ref{A1}--\ref{A3} constants on training data.}
\label{tab:assumptions}
\begin{tabular}{lcccc}
\toprule
Dataset & $\rho_0$ (A1) & $\beta_0$ (A2) & $M_x$ (A3) & $M_c$ (A3) \\
\midrule
XBlock & $0.30$ & $0.38$ & $6.6$ & $23.5$ \\
Reddit & $0.68$ & $0.95$ & $8.6$ & $7.8$ \\
\bottomrule
\end{tabular}
\end{table}

The asymmetry $M_c > M_x$ on XBlock reflects sign-corrected aggregation of positive-mean transaction features under a heavy-tailed degree distribution; Reddit's centered LIWC features yield a balanced pair.

% ============================================================
%  Appendix C — Feature-Context Decomposition
% ============================================================
\section{Additional Baselines}
\label{app:baselines}

\paragraph{Protocol.} Every detector in Tables~\ref{tab:all_detectors} and~\ref{tab:baselines_full} is scored on the same legitimate evaluation nodes and the same attack reserve, with the attacks of \S\ref{sec:threat}. Under drift, models and thresholds stay frozen at $t{=}0$, with the single exception of ARC described below, and the Combined Score follows \S\ref{sec:setup}.

\paragraph{Implementations.} The nine detectors of \S\ref{sec:setup} follow their published configurations, and model selection uses a held-out validation split of the training pool. XGBGraph, BWGNN and Graph Transformer are the GADBench models~\cite{tang2023gadbench}, and GAD-NR, CONAD, AnomalyDAE and OCGNN run through PyGOD~\cite{liu2024pygod}. CONAD and AnomalyDAE exhausted memory on XBlock's 153K nodes under this configuration and are reported on Reddit only. GDN and ARC run from their official repositories. ARC is a pretrained generalist detector applied without fine-tuning, with $300$ legitimate nodes as its in-context normal set. It re-runs its per-graph preprocessing on every evaluated graph, which is its intended usage, while its threshold stays frozen at $t{=}0$. Fixing that preprocessing at $t{=}0$ as well gives Combined Scores of $0.27$ on XBlock and $0.46$ on Reddit.

\paragraph{Ablations.} The signed-cosine score is $-\cos(\bx_i, \bc_i)$ with the sign-corrected context, the feature-context principle with no learning. Context-MLP is a two-layer MLP trained on legitimate pairs to predict $\bx$ from $\bc$ and scored by $\|\bx - \hat{\bx}(\bc)\|_2$, which keeps context conditioning without the free-energy differential or the GLU gate.

\begin{table}[t]
  \caption{All detectors at the last drift step, $T{=}20$.}
  \label{tab:all_detectors}
  \centering
  \small
  \setlength{\tabcolsep}{3.5pt}
  \renewcommand{\arraystretch}{0.92}
  \begin{tabular}{@{}llcccccc@{}}
    \toprule
    & & \multicolumn{3}{c}{XBlock} & \multicolumn{3}{c}{Reddit} \\
    \cmidrule(lr){3-5} \cmidrule(lr){6-8}
    Family & Method & TPR & FPR$_{\mathrm{drift}}$ & Combined & TPR & FPR$_{\mathrm{drift}}$ & Combined \\
    \midrule
    Ours & CB-RBM & $0.99$ & $5.2$ & $\mathbf{0.94}$ & $0.96$ & $8.4$ & $\mathbf{0.88}$ \\
    \midrule
    Supervised & GraphSAGE & $0.89$ & $5.6$ & $0.84$ & $0.06$ & $5.0$ & $0.05$ \\
     & GAT & $0.71$ & $11.6$ & $0.63$ & $0.32$ & $4.6$ & $0.31$ \\
     & GCN & $0.75$ & $5.4$ & $0.71$ & $0.36$ & $4.6$ & $0.35$ \\
     & XGBGraph & $0.98$ & $17.8$ & $0.81$ & $0.01$ & $0.0$ & $0.01$ \\
     & Graph Transformer & $0.87$ & $11.1$ & $0.77$ & $0.07$ & $3.8$ & $0.07$ \\
     & BWGNN & $0.54$ & $5.7$ & $0.51$ & $0.17$ & $8.2$ & $0.16$ \\
    Shift-robust & GDN & $0.77$ & $10.4$ & $0.69$ & $0.24$ & $1.8$ & $0.23$ \\
    Generalist & ARC & $0.40$ & $7.0$ & $0.37$ & $0.86$ & $47.2$ & $0.45$ \\
    Reconstruction & DOMINANT & $0.84$ & $99.0$ & $0.01$ & $0.81$ & $43.2$ & $0.46$ \\
     & GAE & $0.84$ & $66.6$ & $0.28$ & $0.80$ & $91.8$ & $0.07$ \\
     & GAD-NR & $0.53$ & $3.4$ & $0.51$ & $0.67$ & $5.4$ & $0.63$ \\
     & AnomalyDAE & -- & -- & -- & $0.60$ & $82.0$ & $0.11$ \\
    Contrastive & CoLA & $0.19$ & $15.4$ & $0.16$ & $0.56$ & $8.6$ & $0.51$ \\
     & CONAD & -- & -- & -- & $0.79$ & $36.4$ & $0.50$ \\
    One-class & OCGNN & $0.20$ & $7.0$ & $0.18$ & $0.46$ & $7.0$ & $0.43$ \\
    \midrule
    Ablation & Std-RBM (no context) & $0.41$ & $99.8$ & $0.001$ & $0.09$ & $89.4$ & $0.010$ \\
     & Signed cosine (no learning) & $0.18$ & $7.7$ & $0.17$ & $0.86$ & $33.4$ & $0.57$ \\
     & Context-MLP (no RBM/GLU) & $0.76$ & $29.9$ & $0.53$ & $0.92$ & $95.7$ & $0.04$ \\
    \bottomrule
  \end{tabular}
\end{table}

\begin{table}[t]
  \caption{Per-attack AUROC / TPR@5\%FPR (\%).}
  \label{tab:baselines_full}
  \centering
  \scriptsize
  \setlength{\tabcolsep}{3pt}
  \begin{tabular}{@{}llcccc@{}}
    \toprule
    & Method & Isolated Inj. & Feat.\ Camo & Sybil & Rel.\ Camo \\
    \midrule
    \multirow{11}{*}{\rotatebox{90}{\textbf{XBlock}}} & XGBGraph & $99.9$ / $99.3$ & $98.9$ / $93.2$ & $99.9$ / $99.7$ & $99.7$ / $99.9$ \\
     & Graph Transformer & $98.5$ / $94.2$ & $92.7$ / $71.6$ & $98.7$ / $95.2$ & $97.1$ / $86.2$ \\
     & BWGNN & $83.4$ / $28.5$ & $65.1$ / $3.0$ & $99.1$ / $97.1$ & $97.8$ / $88.0$ \\
     & GDN & $95.8$ / $90.3$ & $90.2$ / $79.0$ & $99.1$ / $96.2$ & $65.4$ / $44.3$ \\
     & ARC & $94.8$ / $40.0$ & $94.9$ / $40.0$ & $92.3$ / $40.0$ & $91.0$ / $39.5$ \\
     & GAD-NR & $99.7$ / $100.0$ & $99.4$ / $100.0$ & $60.1$ / $0.0$ & $82.0$ / $12.3$ \\
     & OCGNN & $77.3$ / $34.9$ & $68.9$ / $38.0$ & $55.5$ / $5.1$ & $61.9$ / $0.5$ \\
     & Signed cosine & $83.0$ / $19.5$ & $84.1$ / $23.1$ & $67.2$ / $15.9$ & $70.9$ / $13.3$ \\
     & Context-MLP & $97.1$ / $81.5$ & $96.1$ / $77.4$ & $97.6$ / $90.3$ & $93.4$ / $54.9$ \\
     & AnomalyDAE & -- & -- & -- & -- \\
     & CONAD & -- & -- & -- & -- \\
    \midrule
    \multirow{11}{*}{\rotatebox{90}{\textbf{Reddit}}} & XGBGraph & $59.3$ / $0.0$ & $59.3$ / $0.0$ & $61.8$ / $0.0$ & $64.1$ / $2.5$ \\
     & Graph Transformer & $67.0$ / $0.7$ & $60.7$ / $0.0$ & $66.7$ / $0.2$ & $82.8$ / $28.3$ \\
     & BWGNN & $62.0$ / $7.7$ & $58.0$ / $7.9$ & $57.2$ / $11.0$ & $87.5$ / $42.3$ \\
     & GDN & $67.5$ / $18.8$ & $56.1$ / $9.5$ & $88.3$ / $53.3$ & $67.4$ / $14.0$ \\
     & ARC & $100.0$ / $100.0$ & $100.0$ / $100.0$ & $97.9$ / $92.1$ & $89.7$ / $50.1$ \\
     & GAD-NR & $100.0$ / $100.0$ & $100.0$ / $100.0$ & $91.0$ / $55.5$ & $77.3$ / $11.5$ \\
     & AnomalyDAE & $95.2$ / $88.5$ & $97.0$ / $90.5$ & $81.8$ / $46.5$ & $65.0$ / $15.5$ \\
     & CONAD & $97.7$ / $93.0$ & $88.2$ / $77.5$ & $99.8$ / $100.0$ & $92.0$ / $43.5$ \\
     & OCGNN & $94.5$ / $75.5$ & $96.3$ / $87.0$ & $67.6$ / $0.0$ & $78.3$ / $23.0$ \\
     & Signed cosine & $100.0$ / $100.0$ & $100.0$ / $100.0$ & $90.5$ / $71.5$ & $88.9$ / $71.5$ \\
     & Context-MLP & $99.5$ / $100.0$ & $99.3$ / $100.0$ & $98.1$ / $97.0$ & $95.4$ / $69.0$ \\
    \bottomrule
  \end{tabular}
\end{table}

\section{Robustness Experiments}
\label{app:robust}

\subsection{Contaminated training pool}
\label{app:contam}

We mix fraud into the pool assumed legitimate and refit the sign vector, the model and the threshold on that pool, so the whole pipeline stays label-free. The legitimate nodes are split into training ($80\%$), calibration ($10\%$) and held-out ($10\%$) sets. The natural level adds every non-reserve fraud node once, which is $0.8\%$ of the training pool on XBlock and $1.9\%$ on Reddit, and contaminates the calibration set at the same rate. The $5\%$ and $10\%$ levels oversample fraud to that share of both sets. The clean row is this protocol's own reference, refit on the $80\%$ split, and is not the Table~\ref{tab:benchmark} model. FPR is measured on the held-out legitimate nodes at the threshold calibrated on the contaminated set, and the attacks follow \S\ref{sec:threat}. No sign dimension flips at any level on either dataset, and Isolated Injection, Feature Camouflage and Sybil stay at $100\%$ TPR on XBlock and at least $92\%$ on Reddit.

\begin{table}[t]
  \caption{Fraud in the training and calibration pools (FPR and TPR in \%).}
  \label{tab:contam}
  \centering
  \small
  \setlength{\tabcolsep}{4pt}
  \begin{tabular}{@{}lcccccc@{}}
    \toprule
    & \multicolumn{3}{c}{XBlock} & \multicolumn{3}{c}{Reddit} \\
    \cmidrule(lr){2-4} \cmidrule(lr){5-7}
    Fraud in pool & FPR & AUROC & Rel.\ Camo TPR & FPR & AUROC & Rel.\ Camo TPR \\
    \midrule
    $0\%$ (clean) & $4.6$ & $0.99$ & $80.5$ & $4.0$ & $0.98$ & $65.5$ \\
    Natural & $4.3$ & $0.99$ & $85.6$ & $4.5$ & $0.98$ & $63.5$ \\
    $5\%$ & $3.4$ & $0.99$ & $68.2$ & $4.6$ & $0.98$ & $70.0$ \\
    $10\%$ & $2.7$ & $0.99$ & $39.0$ & $5.0$ & $0.98$ & $66.0$ \\
    \bottomrule
  \end{tabular}
\end{table}

\subsection{Sign vector}
\label{app:sign}

We test a stale sign vector by flipping a random $10\%$, $25\%$ or $50\%$ of its dimensions while the model and the threshold stay fixed at their clean values, and then re-estimate $\mathbf{s}$ from feature statistics on the current unlabeled nodes, exactly as it was first fitted (Table~\ref{tab:sign_flip}). The refresh recovers all $38$ XBlock and all $73$ Reddit dimensions and restores Isolated Injection TPR to $100\%$ and $97.5\%$. We also estimate $\mathbf{s}$ from random samples of unlabeled legitimate nodes, with the model and the threshold again fixed (Table~\ref{tab:sign_small}). Isolated Injection TPR stays at $100\%$ on XBlock and at least $97.5\%$ on Reddit at every sample size. On Reddit, samples of $53$ or $106$ nodes still get some signs wrong in some seeds, so FPR varies widely across seeds, while from $1{,}061$ nodes every seed recovers all signs.

\begin{table}[t]
  \caption{Sign vector with a share of its dimensions flipped, and after re-estimating it without labels (FPR and TPR in \%).}
  \label{tab:sign_flip}
  \centering
  \small
  \setlength{\tabcolsep}{4pt}
  \begin{tabular}{@{}lcccccc@{}}
    \toprule
    & \multicolumn{3}{c}{XBlock} & \multicolumn{3}{c}{Reddit} \\
    \cmidrule(lr){2-4} \cmidrule(lr){5-7}
    Flipped & FPR & Iso.\ TPR & FPR after refresh & FPR & Iso.\ TPR & FPR after refresh \\
    \midrule
    $0\%$ & $5.0$ & $100.0$ & $5.0$ & $5.0$ & $97.5$ & $5.0$ \\
    $10\%$ & $19.1$ & $100.0$ & $5.0$ & $11.6$ & $99.0$ & $5.0$ \\
    $25\%$ & $94.4$ & $100.0$ & $5.0$ & $93.4$ & $98.0$ & $5.0$ \\
    $50\%$ & $99.9$ & $100.0$ & $5.0$ & $97.5$ & $97.5$ & $5.0$ \\
    \bottomrule
  \end{tabular}
\end{table}

\begin{table}[t]
  \caption{Sign vector estimated from $n$ unlabeled legitimate nodes (FPR in \%). $\pm$ is the standard deviation over 5 seeds.}
  \label{tab:sign_small}
  \centering
  \small
  \setlength{\tabcolsep}{4pt}
  \begin{tabular}{@{}rccrcc@{}}
    \toprule
    \multicolumn{3}{c}{XBlock ($38$ dimensions)} & \multicolumn{3}{c}{Reddit ($73$ dimensions)} \\
    \cmidrule(lr){1-3} \cmidrule(lr){4-6}
    $n$ & Wrong dimensions & FPR & $n$ & Wrong dimensions & FPR \\
    \midrule
    $763$ & $1.4 \pm 0.8$ & $6.9 \pm 1.7$ & $53$ & $1.0 \pm 0.9$ & $28.7 \pm 33.8$ \\
    $1{,}526$ & $0.8 \pm 0.8$ & $6.5 \pm 2.8$ & $106$ & $1.0 \pm 0.9$ & $28.7 \pm 33.8$ \\
    $7{,}630$ & $0.6 \pm 0.5$ & $5.1 \pm 0.1$ & $530$ & $0.2 \pm 0.4$ & $6.0 \pm 2.0$ \\
    $15{,}260$ & $0.4 \pm 0.5$ & $5.1 \pm 0.1$ & $1{,}061$ & $0.0 \pm 0.0$ & $5.0 \pm 0.0$ \\
    $152{,}600$ & $0.0 \pm 0.0$ & $5.0 \pm 0.0$ & $10{,}618$ & $0.0 \pm 0.0$ & $5.0 \pm 0.0$ \\
    \bottomrule
  \end{tabular}
\end{table}

\section{Feature-Context Decomposition}
\label{app:decomp}

To identify which signal CB-RBM relies on, we score each test node under four controlled conditions that independently swap features and context between legitimate and adversarial/random sources. Condition~A pairs adversarial features with random context; Condition~B pairs legitimate features with the same random context; Condition~C uses legitimate features with legitimate context (the null condition); Condition~D pairs adversarial features with legitimate context. Results are summarized in Table~\ref{tab:decomp}.

\begin{table}[t]
\centering
\small
\caption{Feature-context decomposition, Combined Score.}
\label{tab:decomp}
\begin{tabular}{lcccc}
\toprule
Condition & Features & Context & XBlock & Reddit \\
\midrule
A & adversarial & random     & $0.942$ & $0.921$ \\
B & legitimate  & random     & $0.942$ & $0.929$ \\
C & legitimate  & legitimate & $0.078$ & $0.095$ \\
D & adversarial & legitimate & $0.111$ & $0.632$ \\
\bottomrule
\end{tabular}
\end{table}

Two regularities appear on both datasets. First, A$\approx$B: random context triggers detection regardless of whether the features are adversarial or legitimate, confirming that CB primarily flags absent context rather than feature deviation. Second, Condition~C is near zero, so legitimate features paired with legitimate context evade detection. The datasets diverge on Condition~D. On XBlock, D$\approx$C: legitimate context alone suppresses the score, and the feature channel contributes nothing measurable. On Reddit, D is substantially above C, so adversarial features with legitimate context are partially flagged. This residual feature sensitivity is consistent with Reddit's richer feature space (73 LIWC and structural features versus XBlock's 38 transaction statistics): the learned weights carry more discriminative feature signal, which surfaces when the context channel is inactive. The main-body interpretation (Section~\ref{sec:discussion}) still holds --- CB-RBM is fundamentally a network-context-presence detector, with feature sensitivity acting only as a secondary channel when the feature space is informative on its own.

% ============================================================
%  Appendix D — Ablation
% ============================================================
\paragraph{Unit-level trace.}\label{app:trace} Table~\ref{tab:trace} follows single nodes through the trained models. A coherent unit has positive feature and gated context pre-activations, so it adds a positive increment to CB, and the last column averages the gated context pre-activation over these units. The legitimate node retains many coherent units and its gated context contribution after $20$ drift steps, so its CB stays high. The injected adversary has random context, the gate shrinks its context contribution, and its CB drops to $28.4$ on XBlock and $33.8$ on Reddit.

\begin{table}[t]
  \caption{Unit-level trace through the trained models.}
  \label{tab:trace}
  \centering
  \small
  \begin{tabular}{@{}llccc@{}}
    \toprule
    Dataset & Node & CB & Coherent units & Mean gated context \\
    \midrule
    \multirow{3}{*}{XBlock} & Legitimate (median CB) & $1046.0$ & $111/256$ & $+7.36$ \\
     & Same node after $20$ drift steps & $952.0$ & $88/256$ & $+6.51$ \\
     & Injected adversary & $28.4$ & $97/256$ & $+0.55$ \\
    \midrule
    \multirow{3}{*}{Reddit} & Legitimate (median CB) & $116.2$ & $57/128$ & $+2.41$ \\
     & Same node after $20$ drift steps & $115.5$ & $52/128$ & $+2.45$ \\
     & Injected adversary & $33.8$ & $34/128$ & $+1.39$ \\
    \bottomrule
  \end{tabular}
\end{table}

\section{Ablation}
\label{app:ablation}

This appendix examines three design choices: the context module itself, the hidden-layer capacity $n_h$, and the neighborhood-pooling operator used to construct the raw context vector $\bc$.

\paragraph{Context module.} Removing the context pathway reduces CB-RBM to Std-RBM, a standard RBM scored by free energy alone. Combining the attack-benchmark TPR (Table~\ref{tab:benchmark}) with the $T{=}20$ drift FPR of Table~\ref{tab:all_detectors} yields a mean Combined Score of $\approx 0.001$ on XBlock and $\approx 0.010$ on Reddit for Std-RBM, compared with $0.94$ and $0.88$ for CB-RBM. The collapse spans two to three orders of magnitude on both datasets and localizes the entire empirical benefit to the context channel.

\paragraph{Hidden-layer capacity.} Figure~\ref{fig:nhidden} shades the min--max envelope of Combined Score across all four attacks as $n_h$ varies from $32$ to $1024$. A tight band at high Combined Score indicates a capacity where every attack is simultaneously well-detected; a wide or low band indicates a regime where at least one attack slips through. The lower boundary of the envelope is essentially Relation Camouflage --- the attack with the largest $\rho_1^{\mathrm{eff}}$ and the smallest measured separation gap (Figure~\ref{fig:cb_gap}). The envelope rises with $n_h$ until the effective coherent-unit count $\beta_0 n_h$ (Assumption~\ref{A2}, measured on training data) saturates, and beyond saturation further capacity only over-fits the legitimate distribution. The two datasets diverge at this saturation point: XBlock's envelope keeps rising through $n_h{=}256$ as the larger pool (${\sim}152$K legitimate nodes) absorbs additional capacity, while Reddit's envelope peaks at $n_h{=}128$ and degrades sharply for $n_h \ge 256$ as Relation Camouflage Combined Score collapses from $0.81$ to $0.33$. The main body accordingly uses $n_h{=}256$ for XBlock and $n_h{=}128$ for Reddit.

\begin{figure}[t]
\centering
\includegraphics[width=0.6\linewidth]{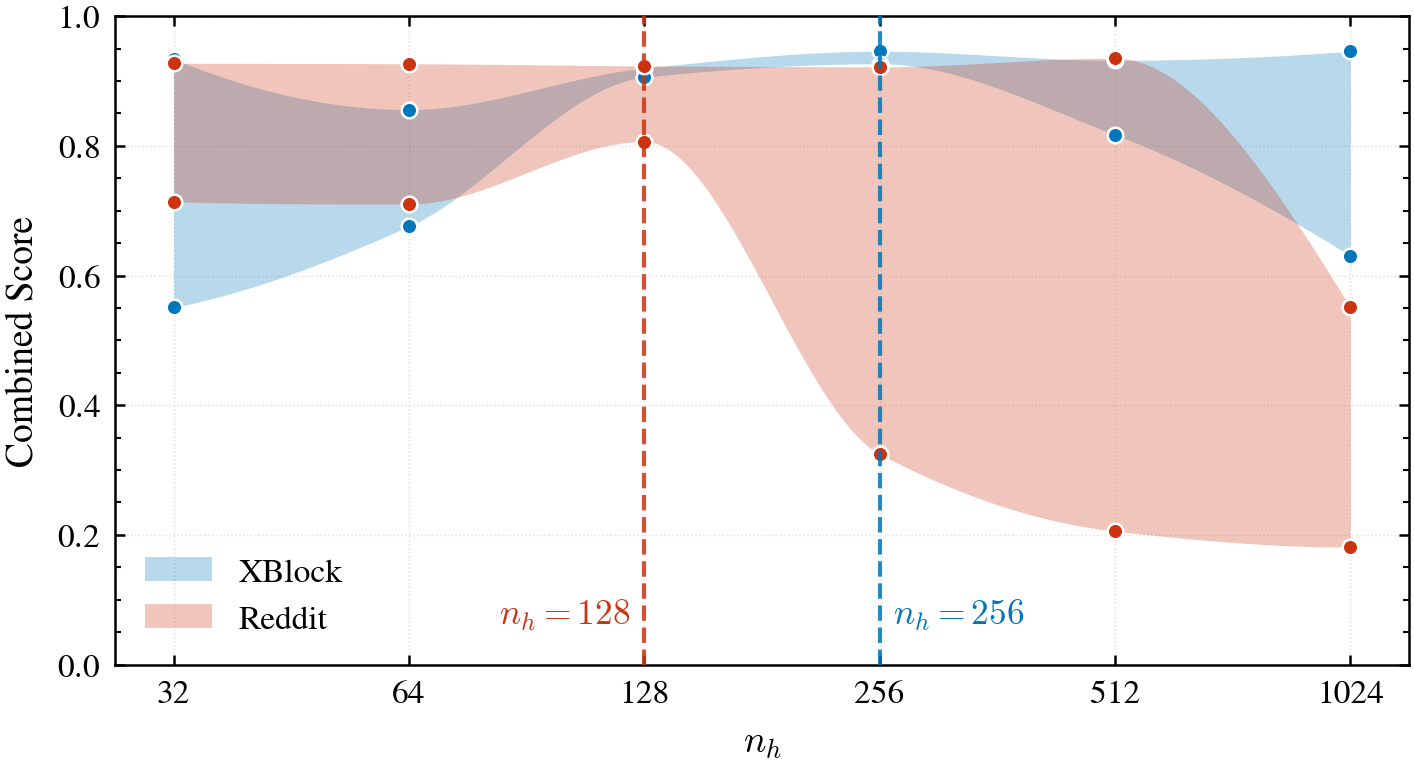}
\caption{Capacity ablation. Shaded band shows min--max Combined Score across all four attacks at each $n_h$. Dashed vertical lines mark the optimal values: $n_h{=}256$ for XBlock, $n_h{=}128$ for Reddit.}
\label{fig:nhidden}
\end{figure}

\begin{table}[t]
\centering
\small
\caption{Pooling operator on XBlock, Combined Score.}
\label{tab:pooling}
\begin{tabular}{lcccc}
\toprule
Attack & Mean & Max & Weighted-deg. & Attention \\
\midrule
Isolated Injection & $0.950$ & $0.914$ & $0.930$ & $0.950$ \\
PGD                & $0.326$ & $0.145$ & $0.224$ & $0.263$ \\
\bottomrule
\end{tabular}
\end{table}

\paragraph{Neighborhood pooling.} The raw context vector $\bc = \mathrm{Pool}\{\bx_u : u \in \mathcal{N}(v)\}$ is constructed by pooling neighbor features. Table~\ref{tab:pooling} compares four operators on XBlock under two attacks chosen to bracket the attack-difficulty spectrum: Isolated Injection (random, unstructured context --- the easiest case for CB-RBM) and PGD (white-box gradient attack perturbing $\bv$ and $\bc$ jointly --- the hardest). We report only this pair because the remaining static attacks in Table~\ref{tab:benchmark} (Feature Camouflage, Sybil, Relation Camouflage) also rely on unstructured, Sybil-generated or hub-borrowed context and track Isolated Injection closely, adding no new information about the pooling operator; the contrast with PGD is what stress-tests the choice. We run this ablation only on XBlock because XBlock uses raw transaction statistics as features, so the pooling operator is the sole mechanism by which neighborhood information enters the model. Reddit features already include engineered graph-derived structural signals (9 of its 73 dimensions), so some neighborhood information is available to the model independently of how $\bc$ is constructed, making the pooling operator less decisive on that dataset. Mean and attention pooling tie at Combined $= 0.95$ on Isolated Injection; max and weighted-degree trail slightly. Under PGD, mean and attention retain a meaningful score while max collapses to $0.15$. All four operators preserve per-dimension alignment (they return a vector in the original feature space rather than a learned embedding), which is what the sign vector $\mathrm{sign}(\mathrm{Pearson}(\bx_d, \bc_d))$ requires to identify the legitimate correlation pattern.

% ============================================================
%  Appendix E — PGD Budget Sweep
% ============================================================
\section{PGD Budget Sweep}
\label{app:pgd}

This appendix evaluates CB-RBM under a white-box threat model: the adversary has gradient access to the detector's weights $\bw$, $\bU$, and the gate, and jointly optimizes $\bv$ and $\bc$ under an $\ell_\infty$-budget. Starting from real adversarial nodes at their natural graph positions, we sweep the perturbation budget $\epsilon \in \{0.01, 0.05, 0.1, 0.2, 0.3, 0.5\}$ with 50 projected-gradient steps. Figure~\ref{fig:pgd} plots detection AUROC on both datasets.

\begin{figure}[t]
\centering
\includegraphics[width=0.6\linewidth]{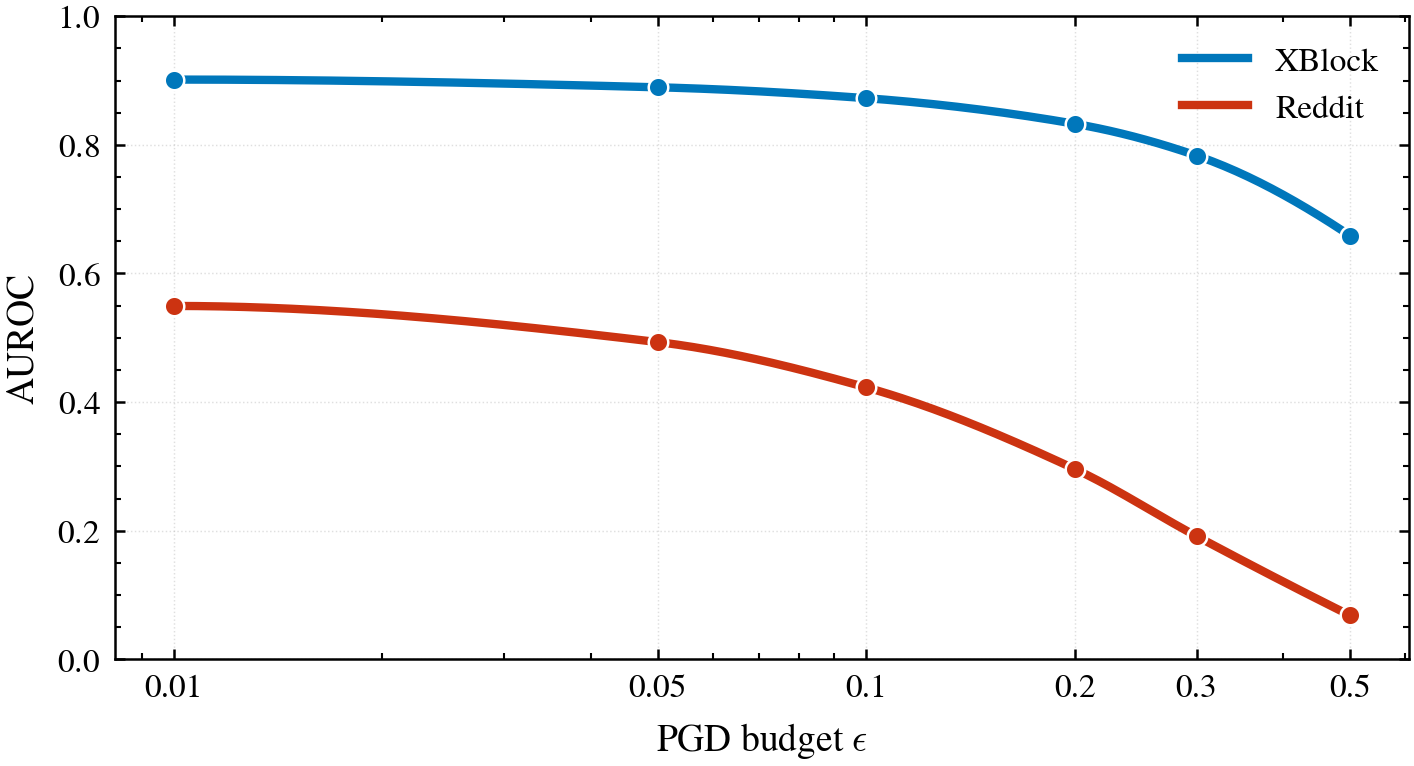}
\caption{White-box PGD budget sweep. Detection AUROC vs.\ perturbation budget $\epsilon$.}
\label{fig:pgd}
\end{figure}

Detection degrades monotonically with $\epsilon$ on both datasets. On XBlock, AUROC falls from $0.90$ at $\epsilon{=}0.01$ to $0.66$ at $\epsilon{=}0.5$, staying above random-chance across the full range. On Reddit, AUROC drops to $0.42$ already at $\epsilon{=}0.1$ and continues to $0.07$ at $\epsilon{=}0.5$; the starting AUROC is lower here because in-graph adversaries reach CB-RBM through the feature channel rather than the context-absence channel. Two factors explain the faster Reddit collapse. First, Reddit's $73$-dimensional feature space gives the PGD optimizer more independent directions in which to move $\bv$ under a fixed $\ell_\infty$ budget than XBlock's $38$-dimensional space. Second, Reddit's higher base correlation $\rho_0 = 0.68$ means the learned GLU gate is tuned to strong context--feature agreement, so small coordinated perturbations to $\bv$ and $\bc$ move the CB score further per unit of perturbation than on XBlock. These results are consistent with the Limitations (\S\ref{sec:discussion}): CB-RBM is designed as a low-cost first-line detector for the common operational case of an adversary without white-box gradient access; combining it with an adversarially trained or certified-robustness second stage would be the natural defense under the PGD threat model.

% ============================================================
%  Appendix F — Drift Stability Extensions
% ============================================================
\section{Drift Stability Extensions}
\label{app:drift_variants}

The main body reports Combined Score under a random-walk drift simulation, following standard practice in the concept-drift literature~\cite{gama2014survey,lu2018learning}. This appendix extends that analysis with alternative drift regimes that vary tails and correlation structure, threshold-level and correlation-weakening drift, real temporal evolution on the XBlock 2017--2018 transaction data, and adversarial-trajectory detection latency (\S\ref{app:drift_variants_sub}--\S\ref{app:latency}).

\subsection{Alternative drift regimes}
\label{app:drift_variants_sub}

At each step $t$, a drift increment $\boldsymbol{\eta}_t$ is added to every node's features and the context vector is recomputed on the updated graph, so drift propagates jointly through node attributes and their aggregated neighborhood --- matching how real distributional shift unfolds on evolving graphs such as the wallet-turnover and user-behavior patterns of the main-body domains. We compare five increment distributions chosen to cover the natural-drift regimes a detector is likely to encounter in deployment: \emph{Gaussian} --- isotropic bounded-variance drift in the same family as the main-body simulation; \emph{correlated $\rho{=}0.3$} and \emph{correlated $\rho{=}0.6$} --- coordinated population shifts where subsets of features co-move, modeled with equicorrelated Gaussian increments; and \emph{heavy-tail $\mathrm{df}{=}10$} / \emph{heavy-tail $\mathrm{df}{=}3$} --- bursty, outlier-driven drift modeled with Student-$t$ increments, smaller $\mathrm{df}$ yielding heavier tails. Table~\ref{tab:drift_variants} reports the drift-induced false-positive rate at $T{=}20$ for CB-RBM and Std-RBM on both datasets. The per-step variance budget here is calibrated for cross-regime comparison and differs from the headline drift simulation; the qualitative pattern --- Std-RBM collapse, CB-RBM stability --- is unchanged.

\begin{table}[t]
\centering
\small
\caption{Drift-induced FPR at $T{=}20$ under five drift regimes (\%).}
\label{tab:drift_variants}
\begin{tabular}{lcccc}
\toprule
& \multicolumn{2}{c}{XBlock} & \multicolumn{2}{c}{Reddit} \\
\cmidrule(lr){2-3} \cmidrule(lr){4-5}
Regime & CB-RBM & Std-RBM & CB-RBM & Std-RBM \\
\midrule
Gaussian                       & $8.4$  & $96.4$ & $9.0$  & $73.8$ \\
Correlated $\rho{=}0.3$        & $12.8$ & $95.8$ & $2.8$  & $74.0$ \\
Correlated $\rho{=}0.6$        & $15.4$ & $97.2$ & $0.2$  & $69.0$ \\
Heavy-tail $\mathrm{df}{=}10$  & $10.8$ & $95.2$ & $9.0$  & $71.2$ \\
Heavy-tail $\mathrm{df}{=}3$   & $7.6$  & $73.8$ & $6.6$  & $55.2$ \\
\bottomrule
\end{tabular}
\end{table}

CB-RBM stays at $\leq 16\%$ FPR across every regime and dataset, while Std-RBM exceeds $55\%$ in every regime and $>95\%$ in four of the five XBlock regimes. Two observations are consistent with Theorem~\ref{thm:drift}. First, the Wasserstein--Lipschitz bound is \emph{distribution-agnostic}: it depends only on $W_1(\mathcal{P}_t, \mathcal{P}_{t-1}) \le \delta_{\mathrm{drift}}$, not on the tails or correlation structure of the drift increment, consistent with CB-RBM behaving similarly across all five variants (the slightly lower df=3 row reflects Student-t's variance-matched mass concentrating near zero, so most nodes drift less than under Gaussian rather than more). Second, the two datasets respond oppositely to correlated drift by the same mechanism. Drift aligned with the learned $(\bv,\bc)$ correlation lies in a near-flat valley of the CB surface --- the softplus differences in Eq.~\eqref{eq:cb} nearly cancel --- so FPR drops; drift off this axis avoids the valley and FPR rises. Reddit's strong $\rho_0{=}0.68$ aligns closely with the equicorrelated $\rho{=}0.6$ increment, driving FPR toward zero ($9.0\!\to\!0.2\%$); XBlock's weaker $\rho_0{=}0.30$ lies further off-axis, raising FPR modestly ($8.4\!\to\!15.4\%$, still far below Std-RBM at every regime). This is consistent with the ``conservative bound'' discussion of Section~\ref{sec:discussion}: real drift's CB-impact depends on alignment with the learned correlation, while the worst-case Theorem~\ref{thm:drift} bound holds in either direction.

\paragraph{RMS-based numerical comparison.} Table~\ref{tab:lipschitz} grounds the slack-factor numbers reported in \S\ref{sec:discussion}. $L$ is the Lipschitz expression of Lemma~\ref{lem:lip} evaluated at the RMS value of $M_c$ (Appendix~\ref{app:datasets}); $\delta_{\mathrm{drift}}$ is the per-step Wasserstein-1 budget; observed $|\Delta\CB|$ is the realized mean-CB shift on legitimate nodes after $T$ Gaussian-drift steps. The slack factor (bound$/$observed) grows with $T$ on both datasets and is consistent with the conservative-bound argument: the worst-case $L\cdot T \cdot\delta_{\mathrm{drift}}$ allows every coordinate to move adversarially, while real drift moves $(\bv,\bc)$ jointly along the legitimate-correlation axis where the softplus differences nearly cancel.

\begin{table}[t]
\centering
\small
\caption{Wasserstein--Lipschitz bound vs.\ observed CB shift on legitimate nodes (Gaussian drift). $L$ is the Lipschitz constant of $\CB$ evaluated at the RMS value of $M_c$ (Appendix~\ref{app:datasets}); $\delta_{\mathrm{drift}}$ is the per-step $W_1$ budget; the RMS-based budget is $L\cdot T\cdot\delta_{\mathrm{drift}}$; slack is budget$/$observed.}
\label{tab:lipschitz}
\begin{tabular}{lcccccl}
\toprule
Dataset & $L$ & $\delta_{\mathrm{drift}}$ & $T$ & RMS-based budget & Observed $|\Delta\CB|$ & Slack \\
\midrule
\multirow{4}{*}{XBlock} & \multirow{4}{*}{$1189.5$} & \multirow{4}{*}{$1.23$}
                       & $5$  & $7{,}288$  & $23.3$ & $312\times$ \\
                       &     &     & $10$ & $14{,}575$ & $35.7$ & $408\times$ \\
                       &     &     & $15$ & $21{,}863$ & $46.1$ & $474\times$ \\
                       &     &     & $20$ & $29{,}151$ & $55.4$ & $\mathbf{526\times}$ \\
\midrule
\multirow{4}{*}{Reddit} & \multirow{4}{*}{$154.0$} & \multirow{4}{*}{$1.70$}
                       & $5$  & $1{,}313$  & $0.68$ & $1{,}921\times$ \\
                       &     &     & $10$ & $2{,}625$  & $1.03$ & $2{,}549\times$ \\
                       &     &     & $15$ & $3{,}938$  & $1.17$ & $3{,}373\times$ \\
                       &     &     & $20$ & $5{,}250$  & $1.16$ & $\mathbf{4{,}540\times}$ \\
\bottomrule
\end{tabular}
\end{table}
\subsection{Threshold-level drift}
\label{app:thrdrift}

Evaluating $L\,T\,\delta_{\mathrm{drift}}$ with the RMS-based $L$ gives $29{,}151$ on XBlock and $5{,}250$ on Reddit at $T{=}20$ (Table~\ref{tab:lipschitz}). The threshold is instead set by the $95$th percentile of the legitimate scores, which we track in a separate three-seed run with the model and the threshold frozen at $t{=}0$ (Table~\ref{tab:thrdrift}). At $T{=}20$ this percentile shifts by $42.7$ on XBlock and $3.4$ on Reddit, about a thousandth of that budget, and mean FPR at the frozen threshold rises from $5\%$ to $6.1\%$ and $7.1\%$. Using the same RMS-based value of $L$, the horizon expression of Corollary~\ref{cor:horizon} evaluates to $0.32$ steps on XBlock and $0.12$ on Reddit, while mean FPR at the frozen $95$th-percentile threshold stays within $2.1$ points of the target at every reported step through $T{=}20$.

\begin{table}[t]
  \caption{Shift of the $95$th-percentile threshold and FPR (\%) at the threshold frozen at $t{=}0$. $\pm$ is the standard deviation over 3 drift seeds.}
  \label{tab:thrdrift}
  \centering
  \small
  \begin{tabular}{@{}rcccc@{}}
    \toprule
    & \multicolumn{2}{c}{XBlock} & \multicolumn{2}{c}{Reddit} \\
    \cmidrule(lr){2-3} \cmidrule(lr){4-5}
    $T$ & Threshold shift & FPR & Threshold shift & FPR \\
    \midrule
    $1$ & $5.6 \pm 1.9$ & $5.1 \pm 0.1$ & $0.18 \pm 0.01$ & $5.1 \pm 0.0$ \\
    $5$ & $14.1 \pm 4.2$ & $5.4 \pm 0.1$ & $0.84 \pm 0.23$ & $5.5 \pm 0.1$ \\
    $10$ & $27.3 \pm 5.8$ & $5.6 \pm 0.1$ & $1.67 \pm 0.31$ & $6.0 \pm 0.1$ \\
    $15$ & $34.2 \pm 6.6$ & $5.8 \pm 0.1$ & $2.55 \pm 0.27$ & $6.4 \pm 0.2$ \\
    $20$ & $42.7 \pm 5.9$ & $6.1 \pm 0.1$ & $3.41 \pm 0.27$ & $7.1 \pm 0.3$ \\
    \bottomrule
  \end{tabular}
\end{table}

\subsection{Correlation-weakening drift}
\label{app:harddrift}

The drift of Figure~\ref{fig:drift_sim} moves features and context together. We weaken that link in three ways at $T{=}20$, with the model and the threshold frozen at $t{=}0$. Context lag recomputes each node's context from features $L$ steps behind its own. Private drift adds to each node an independent drift that its neighbors' context does not include, scaled by $p$ relative to the shared drift. Reversal negates each node's drift relative to its neighborhood on a fraction of the dimensions, so features and context anti-align there (Table~\ref{tab:harddrift}). Recalibrating the threshold at the $95$th percentile of the drifted pool's scores, the label-free refresh of Algorithm~\ref{alg:cbrbm}, returns FPR to $5\%$ by construction, and under reversal it keeps Isolated Injection TPR at $100\%$ on XBlock and at least $95.5\%$ on Reddit. Table~\ref{tab:remedy} applies the same recalibration to the eight methods of Table~\ref{tab:benchmark} under the co-evolving drift of Figure~\ref{fig:drift_sim} and scores the four attacks against the recalibrated threshold.

\begin{table}[t]
  \caption{Correlation-weakening drift at $T{=}20$, measured $\rho$ and legitimate FPR (\%) at the frozen threshold. Each block has its own co-evolving reference run.}
  \label{tab:harddrift}
  \centering
  \small
  \begin{tabular}{@{}lcccc@{}}
    \toprule
    & \multicolumn{2}{c}{XBlock} & \multicolumn{2}{c}{Reddit} \\
    \cmidrule(lr){2-3} \cmidrule(lr){4-5}
    Regime & $\rho$ & FPR & $\rho$ & FPR \\
    \midrule
    Co-evolving (reference) & $0.141$ & $5.4$ & $0.615$ & $7.5$ \\
    Context lag, $L{=}5$ & $0.141$ & $5.3$ & $0.588$ & $8.2$ \\
    Context lag, $L{=}10$ & $0.142$ & $5.3$ & $0.567$ & $9.4$ \\
    Context lag, $L{=}20$ & $0.144$ & $5.6$ & $0.509$ & $10.6$ \\
    Private drift, $p{=}0.3$ & $0.138$ & $5.4$ & $0.602$ & $8.0$ \\
    Private drift, $p{=}0.6$ & $0.125$ & $5.7$ & $0.568$ & $9.4$ \\
    Private drift, $p{=}1.0$ & $0.113$ & $6.2$ & $0.506$ & $13.2$ \\
    \midrule
    Co-evolving (reference) & $0.146$ & $6.1$ & $0.613$ & $7.8$ \\
    Reversed, $10\%$ of dims & $0.145$ & $6.1$ & $0.581$ & $8.1$ \\
    Reversed, $25\%$ of dims & $0.144$ & $6.2$ & $0.534$ & $8.5$ \\
    Reversed, $50\%$ of dims & $0.141$ & $6.1$ & $0.464$ & $10.1$ \\
    \bottomrule
  \end{tabular}
\end{table}

\begin{table}[t]
  \caption{Threshold recalibration for the eight methods of Table~\ref{tab:benchmark} under the co-evolving drift at $T{=}20$, Combined Score before and after, from this experiment's run. Best refreshed score in bold.}
  \label{tab:remedy}
  \centering
  \small
  \begin{tabular}{@{}lcccc@{}}
    \toprule
    & \multicolumn{2}{c}{XBlock} & \multicolumn{2}{c}{Reddit} \\
    \cmidrule(lr){2-3} \cmidrule(lr){4-5}
    Method & Frozen & Refreshed & Frozen & Refreshed \\
    \midrule
    CB-RBM & $0.92$ & $\mathbf{0.90}$ & $0.88$ & $\mathbf{0.90}$ \\
    Std-RBM & $0.003$ & $0.004$ & $0.016$ & $0.014$ \\
    GraphSAGE & $0.82$ & $0.82$ & $0.06$ & $0.06$ \\
    GAT & $0.60$ & $0.61$ & $0.11$ & $0.11$ \\
    GCN & $0.74$ & $0.72$ & $0.15$ & $0.15$ \\
    DOMINANT & $0.01$ & $0.02$ & $0.53$ & $0.64$ \\
    CoLA & $0.19$ & $0.07$ & $0.16$ & $0.13$ \\
    GAE & $0.32$ & $0.01$ & $0.05$ & $0.48$ \\
    \bottomrule
  \end{tabular}
\end{table}

\paragraph{Topology change.} We also change the edges at $T{=}20$, on top of the feature drift, with the model and the threshold frozen at $t{=}0$. Growth keeps every edge and adds $10\%$, $30\%$ or $50\%$ new triadic edges, the pattern of cumulative interaction graphs such as ours, where on-chain edges are permanent and Reddit interactions accumulate. It raises the legitimate FPR to $12.6$, $24.7$ and $37.3\%$ on XBlock and $12.9$, $24.6$ and $34.6\%$ on Reddit. The label-free threshold refresh restores $5.0\%$ on both, with Isolated Injection TPR at $100\%$ on XBlock and $95.5$, $93.0$ and $91.0\%$ on Reddit. Random rewiring of $10\%$ of edges, evaluated as a topology stress test, raises the FPR to $20.4\%$ and $31.7\%$, and the refresh restores $5.0\%$ with Isolated Injection TPR at $76.4\%$ and $89.5\%$.

\subsection{Real temporal validation (XBlock)}
\label{app:temporal_real}

To check that synthetic drift is not the only regime where CB-RBM holds up, we train on XBlock data through 2017 (a dataset with natively timestamped edges) and evaluate on five cumulative windows, the 2017 training window and 2018-Q1 through Q4. Between windows the Ethereum ecosystem evolves organically --- new wallets appear, transaction patterns shift, and each window adds roughly $200$K--$500$K new edges over the preceding one --- producing real natural drift rather than a simulated increment. Figure~\ref{fig:temporal} plots FPR and TPR at the calibrated $5\%$-FPR operating point for all eight methods across the five windows. Node features in each window use transactions up to its cutoff only, and the scaler, sign vector and threshold are fitted on the training window and then frozen. Rebuilding each window's context from pre-cutoff edges only gives FPR $5.0$, $3.8$, $3.4$, $3.2$ and $3.1\%$ across the five windows, against $5.0$, $3.5$, $3.2$, $3.2$ and $3.1\%$ for the construction used in Figure~\ref{fig:temporal}.

\begin{figure}[t]
\centering
\includegraphics[width=0.48\linewidth]{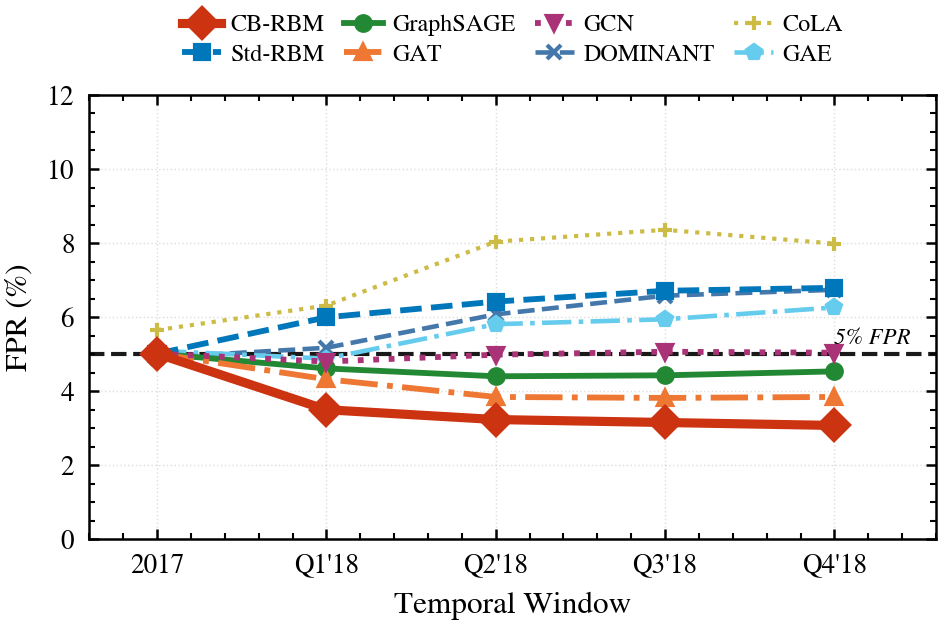}\hfill
\includegraphics[width=0.48\linewidth]{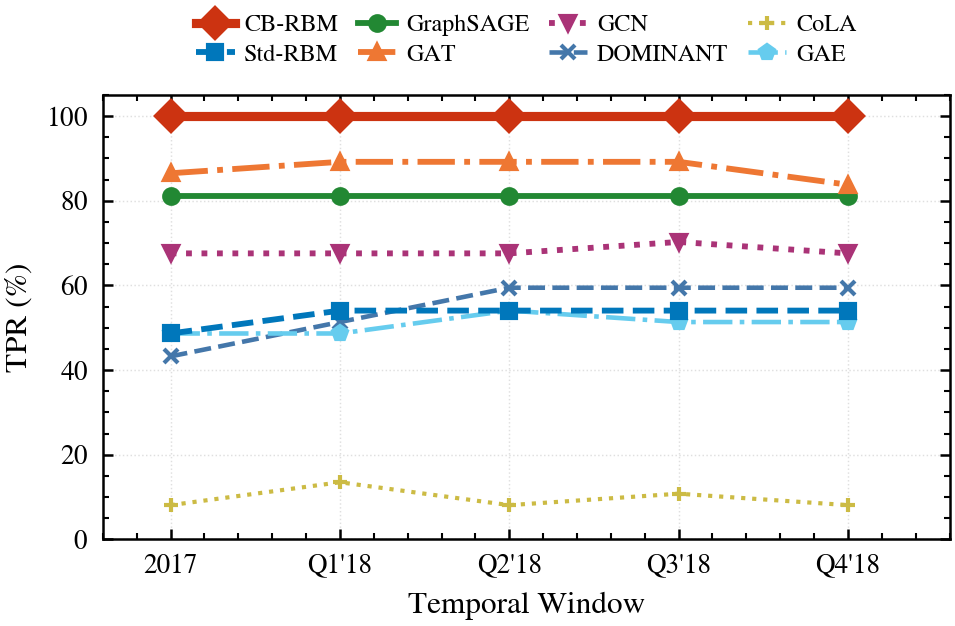}
\caption{XBlock real-temporal validation. (Left) FPR and (right) TPR at the calibrated $5\%$-FPR operating point across five windows (trained and calibrated on the 2017 window, evaluated on 2017--2018 windows).}
\label{fig:temporal}
\end{figure}

CB-RBM holds at TPR$=100\%$ across all five windows; the runner-up (GAT) holds around $84\%$--$89\%$. The feature distribution drifts monotonically --- $\|\boldsymbol{\mu}_t - \boldsymbol{\mu}_0\|$ grows from $0$ to $\approx 73$ over the windows --- without translating into an FPR increase or TPR drop for CB-RBM. The supervised GNN baselines are also stable at lower absolute levels ($\approx 70\%$--$89\%$), while Std-RBM and GAE hold flat in the $45$--$55\%$ band, DOMINANT trends slightly higher (up to $\approx 60\%$), and CoLA stays below $16\%$. Real temporal evolution therefore reproduces the pattern of the synthetic drift simulation. CB-RBM and the supervised GNNs remain near the calibration target, while the reconstruction and density baselines show increased false alarms.

\subsection{Detection latency}
\label{app:latency}

To complement the aggregate FPR view, Table~\ref{tab:latency} reports the trajectory step $T$ at which each detector first fires on a five-scenario suite: one benign control (natural drift, which should \emph{not} trigger an alarm) and the four adversarial scenarios of Table~\ref{tab:benchmark}. The setup evolves $200$ target nodes over $T{=}0,\ldots,50$ under each scenario; the \emph{first-fire} step is the earliest $T$ at which at least $90\%$ of those nodes exceed the detection threshold. ``Never'' means the $90\%$ condition is not triggered within $T{=}50$ steps.

\begin{table}[t]
\centering
\small
\caption{First-fire step $T$ (first $T$ with $\geq\!90\%$ of $200$ target nodes flagged). Lower is better for adversarial rows; ``never'' is correct for natural drift.}
\label{tab:latency}
\begin{tabular}{lcccc}
\toprule
& \multicolumn{2}{c}{XBlock} & \multicolumn{2}{c}{Reddit} \\
\cmidrule(lr){2-3} \cmidrule(lr){4-5}
Scenario & CB-RBM & Std-RBM & CB-RBM & Std-RBM \\
\midrule
Natural drift (benign)   & never & $15$   & never & $23$   \\
\midrule
Sybil                    & $0$   & never  & $0$   & never  \\
Relation Camouflage      & $0$   & never  & $0$   & never  \\
Isolated Injection       & $18$  & never  & $24$  & never  \\
Feature Camouflage       & $24$  & never  & $26$  & never  \\
\bottomrule
\end{tabular}
\end{table}

    The pattern is the dual of the drift-stability view. Std-RBM false-alarms on natural drift at $T{=}15$ (XBlock) / $T{=}23$ (Reddit) and is blind to every adversarial trajectory: the free-energy channel alone cannot distinguish the attacks from legitimate drift once context is excluded. CB-RBM shows the opposite pattern, staying silent under natural drift and firing on every adversarial trajectory. Sybil and Relation Camouflage are fixed-state attacks (the clique injection or hub rewire is present from $T{=}0$), and CB-RBM fires immediately. Isolated Injection and Feature Camouflage are evolving trajectories: context interpolates from legitimate at $T{=}0$ toward random by $T{=}50$, so CB-RBM (correctly) does not fire while the context still resembles legitimate input, and the $T{=}18$--$26$ first-fire step marks the crossing where the context has decayed enough to become distinguishable. Two-channel latency therefore separates cleanly: the feature-only model fires on the wrong trajectory, while the context-aware model fires on the right ones.

% ============================================================
%  Appendix G — Context Interpolation Sweep
% ============================================================
\paragraph{Growing Sybil.}\label{app:sybil_grow} The Sybil attack of Table~\ref{tab:benchmark} uses a clique of fixed size from the start. Here the clique size grows from $2$ to $50$ over $50$ steps. At each step the clique is rebuilt at its current size and the context of the $195$ XBlock and $200$ Reddit attack nodes connected to it is recomputed, while each detector keeps the threshold calibrated at $t{=}0$ (Table~\ref{tab:sybil_grow}). With two members CB-RBM already flags $92.3\%$ of the attack nodes on XBlock and $61.0\%$ on Reddit, crosses $90\%$ on Reddit by step $4$ (between the rows of the table) and flags $100\%$ on both datasets by step $10$. Std-RBM stays at $51.3\%$ and $9.5\%$ throughout.

\begin{table}[t]
  \caption{Fraction of attack nodes flagged (\%) as the Sybil clique grows from $2$ to $50$ members.}
  \label{tab:sybil_grow}
  \centering
  \small
  \begin{tabular}{@{}rrcccc@{}}
    \toprule
    & & \multicolumn{2}{c}{XBlock} & \multicolumn{2}{c}{Reddit} \\
    \cmidrule(lr){3-4} \cmidrule(lr){5-6}
    Step & Clique size & CB-RBM & Std-RBM & CB-RBM & Std-RBM \\
    \midrule
    $0$ & $2$ & $92.3$ & $51.3$ & $61.0$ & $9.5$ \\
    $10$ & $12$ & $100.0$ & $51.3$ & $100.0$ & $9.5$ \\
    $20$ & $21$ & $100.0$ & $51.3$ & $100.0$ & $9.5$ \\
    $30$ & $31$ & $100.0$ & $51.3$ & $100.0$ & $9.5$ \\
    $40$ & $40$ & $100.0$ & $51.3$ & $100.0$ & $9.5$ \\
    $50$ & $50$ & $100.0$ & $51.3$ & $100.0$ & $9.5$ \\
    \bottomrule
  \end{tabular}
\end{table}

\section{Context Interpolation Sweep}
\label{app:taxonomy}

We complement Table~\ref{tab:benchmark}'s discrete attacks with a \emph{context interpolation sweep}: a continuous-intensity extension of the camouflage concept introduced in the fraud-detection literature~\cite{hooi2016fraudar}. An attacker parameter $\alpha \in [0,1]$ interpolates between the more \emph{revealing} and the more \emph{camouflaging} context for the dataset:
\begin{equation}
c_\alpha \;=\; (1-\alpha)\,c_{\text{revealed}} + \alpha\,c_{\text{camo}}, \qquad \alpha \in [0,1],
\end{equation}
where $c_{\text{camo}}$ is whichever of $\{c_{\text{atk}},\,\bar{c}_{\text{legit}}\}$ has $\rho_1^{\mathrm{eff}}$ closer to $\rho_0$ on that dataset (the harder camouflage), and $c_{\text{revealed}}$ is the other. On XBlock $\rho_1^{\mathrm{natural}}=0.22<\rho_0$, so $c_{\text{camo}}=\bar{c}_{\text{legit}}$ and $c_{\text{revealed}}=c_{\text{atk}}$; on Reddit $\rho_1^{\mathrm{natural}}=0.71>\rho_0$, so the roles swap --- $c_{\text{camo}}=c_{\text{atk}}$ and $c_{\text{revealed}}=\bar{c}_{\text{legit}}$. The sweep varies the attacker's effective correlation $\rho_1^{\mathrm{eff}}(\alpha)$ continuously, and with it the measured gap $|\Delta\rho^{\mathrm{eff}}(\alpha)| = |\rho_0 - \rho_1^{\mathrm{eff}}(\alpha)|$, letting us read off the boundary at which the CB-RBM detection margin closes.

\begin{figure}[t]
\centering
\begin{subfigure}[t]{0.44\linewidth}
\centering
\includegraphics[width=\linewidth]{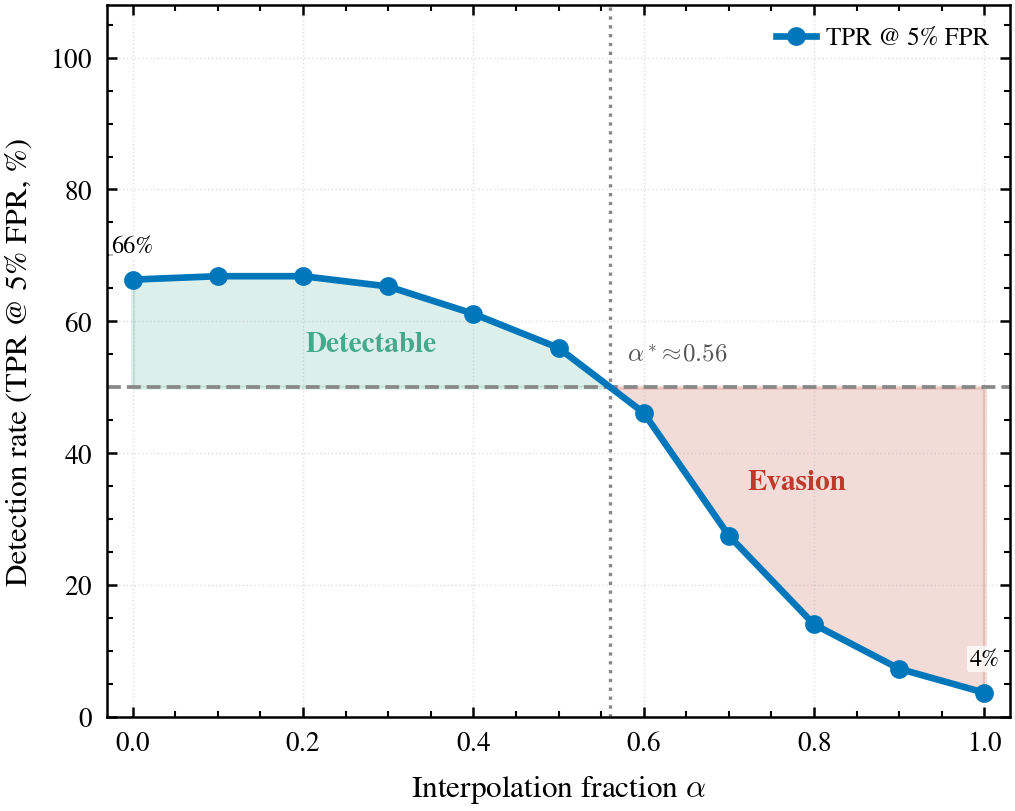}
\caption{XBlock: detection degrades as $\alpha \to 1$.}
\end{subfigure}\hfill
\begin{subfigure}[t]{0.44\linewidth}
\centering
\includegraphics[width=\linewidth]{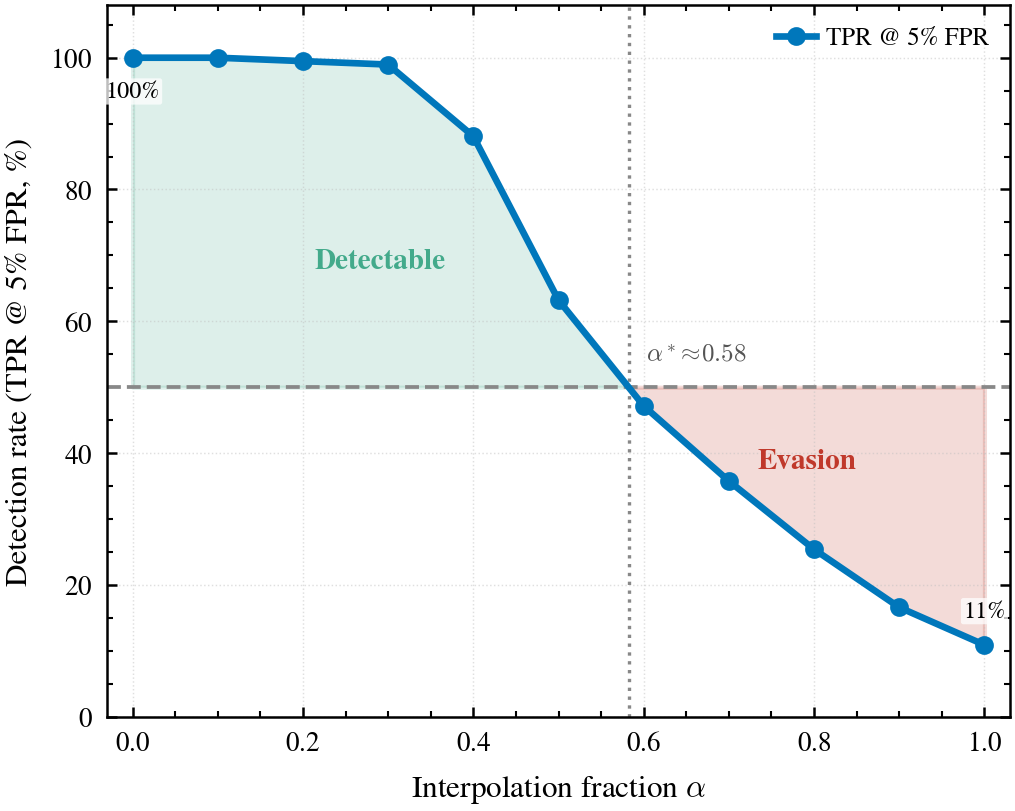}
\caption{Reddit: detection degrades as $\alpha \to 1$.}
\end{subfigure}
\caption{Context interpolation sweep, detection view. TPR at the calibrated $5\%$-FPR operating point vs $\alpha$; the $50\%$ crossing marks $\alpha^*$.}
\label{fig:alpha-det}
\end{figure}

\begin{figure}[t]
\centering
\begin{subfigure}[t]{0.45\linewidth}
\centering
\includegraphics[width=\linewidth]{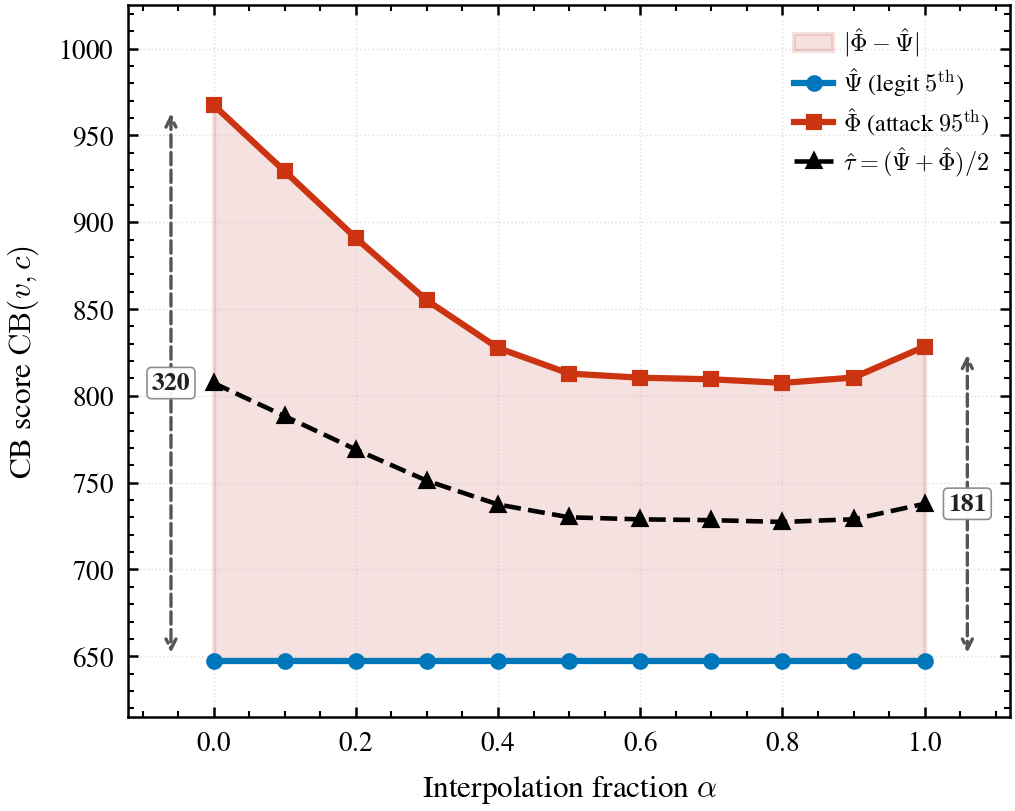}
\caption{XBlock: $\hat\Phi \to \hat\Psi$ as $\alpha \to 1$.}
\end{subfigure}\hfill
\begin{subfigure}[t]{0.45\linewidth}
\centering
\includegraphics[width=\linewidth]{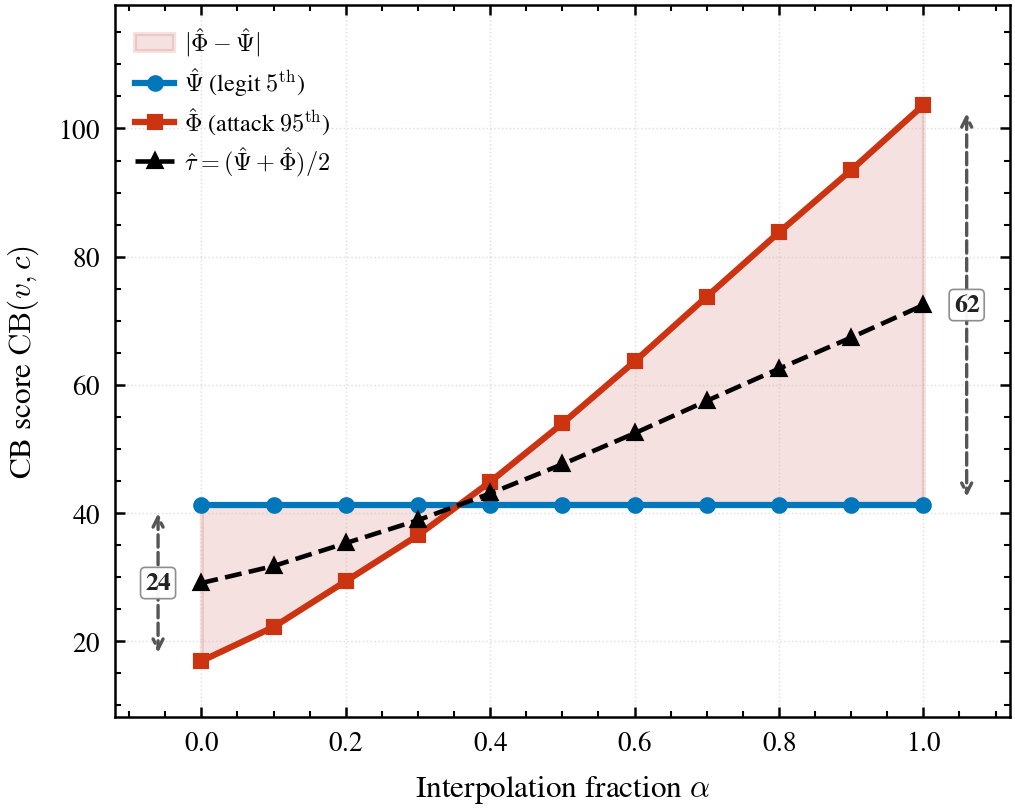}
\caption{Reddit: $\hat\Phi \to \hat\Psi$ as $\alpha \to 1$.}
\end{subfigure}
\caption{Context interpolation sweep, separation view. Tail-band quantities $\hat\Psi$ (legit $5^{\rm th}$-pctile), $\hat\Phi$ (attack $95^{\rm th}$-pctile), and empirical tail-band midpoint $\hat\tau=(\hat\Psi+\hat\Phi)/2$; shaded band is $|\hat\Phi-\hat\Psi|$, with magnitudes at $\alpha{=}0,1$ annotated.}
\label{fig:alpha-gap}
\end{figure}

In Figures~\ref{fig:alpha-det} and~\ref{fig:alpha-gap}, $\alpha{=}1$ denotes the hardest camouflage for that dataset (legit-mean context on XBlock, natural attack context on Reddit). This works because Relation Camouflage is a mirror operation in which the attacker picks whichever context distribution has $\rho_1^{\mathrm{eff}}$ closest to $\rho_0$, minimizing $|\Delta\rho^{\mathrm{eff}}|$.

\emph{XBlock} ($\rho_0{=}0.30$, $\rho_1^{\mathrm{natural}}{=}0.22$). Legit context is \emph{more} coherent than the attacker's own neighborhood, so injecting legit-mean context is the camouflage move. At $\alpha{=}0$ TPR is $66\%$ with $|\hat\Phi-\hat\Psi| \approx 320$; at $\alpha{=}0.5$ TPR is $56\%$; at $\alpha{=}1$ TPR collapses to $4\%$ with $|\hat\Phi-\hat\Psi| \approx 181$; crossing at $\alpha^* \approx 0.56$.

\emph{Reddit} ($\rho_0{=}0.68$, $\rho_1^{\mathrm{natural}}{=}0.71$). The attacker's own context is already as coherent as the legit one, so the camouflage move is to keep attack context and \emph{not} mix in legit-mean. At $\alpha{=}0$ (pure legit-mean context) TPR is $100\%$ with $\hat\Phi$ well below $\hat\Psi$; $\alpha{=}0.5$ gives TPR $63\%$; at $\alpha{=}1$ (pure attack context) TPR is $11\%$ with $\hat\Phi$ crossing above $\hat\Psi$; crossing at $\alpha^* \approx 0.58$.

\paragraph{Consequence.} The ``hardest'' camouflage is not a fixed recipe but the context distribution that minimizes $|\Delta\rho^{\mathrm{eff}}|$ for the given dataset. For XBlock, legit-mean substitution minimizes $|\Delta\rho^{\mathrm{eff}}|$; for Reddit, the attacker's own natural context already achieves this. In both regimes, CB-RBM's failure point is tied to a single scalar condition: the collapse of $|\Delta\rho^{\mathrm{eff}}|$, consistent with the $\rho_1$-dependence of the Theorem~\ref{thm:sep} margin.

% ============================================================
%  Appendix H — Generalization to Additional Datasets
% ============================================================
\section{Generalization across Domains}
\label{app:generalisation}
CB-RBM's two main-body datasets cover finance (XBlock) and social media (Reddit). This appendix tests whether the feature-context consistency principle transfers to further domains with different graph properties, label distributions, and fraud mechanics. Two additional organic datasets are reported: crowdsourcing (Tolokers)~\cite{platonov2023critical} and e-commerce review fraud (YelpChi)~\cite{rayana2015collective}.

\subsection{Tolokers (Crowdsourcing)}
\label{app:tolokers}

Tolokers~\cite{platonov2023critical} is a crowdsourcing graph with $11{,}758$ worker nodes, $1{,}038{,}000$ edges, $10$ base features augmented to $19$ with the same graph-structural signals used for Reddit (Appendix~\ref{app:datasets}), and a $21.8\%$ unreliable-worker rate (organic labels). After sign correction $\rho_0 = 0.215$, placing it below XBlock ($0.30$) at the lower end of the homophily range we test, with a positive raw correlation gap. Tolokers is substantially denser than the main-body graphs (average degree $88$ vs.\ Reddit's $15$).

\begin{table}[t]
  \caption{Attack benchmark on Tolokers. Each cell: AUROC / TPR@5\%FPR (\%). Best per row in \textbf{bold}.}
  \label{tab:tolokers_benchmark}
  \centering
  \setlength{\tabcolsep}{2pt}
  \renewcommand{\arraystretch}{1.05}
  \begin{tabular}{@{}lcccccccc@{}}
    \toprule
    Attack & \multicolumn{2}{c}{Energy-based} & \multicolumn{3}{c}{Supervised} & \multicolumn{3}{c}{Unsupervised} \\
    \cmidrule(lr){2-3} \cmidrule(lr){4-6} \cmidrule(lr){7-9}
    & CB-RBM & Std-RBM & SAGE & GAT & GCN & DOM. & CoLA & GAE \\
    \midrule
    Isolated Inj. & \textbf{96.4}/\textbf{89.3} & 51.2/4.2 & 88.4/52.5 & 72.0/34.0 & 92.8/58.7 & 73.3/23.7 & 55.0/26.2 & 92.7/77.8 \\
    Feat.\ Camo   & \textbf{97.5}/\textbf{87.8} & 51.2/6.0 & 66.7/28.0 & 54.4/18.7 & 70.0/34.3 & 77.9/33.5 & 57.6/38.5 & 95.4/86.5 \\
    Sybil         & 98.9/96.7 & 51.2/4.2 & 97.1/86.7 & 98.9/\textbf{100} & 99.1/\textbf{100} & 97.3/96.7 & 68.1/40.5 & \textbf{99.8}/\textbf{100} \\
    Rel.\ Camo    & \textbf{97.2}/\textbf{79.0} & 51.2/4.2 & 79.4/13.7 & 80.1/22.3 & 77.6/15.7 & 68.6/15.5 & 63.8/11.7 & 57.6/12.7 \\
    \bottomrule
  \end{tabular}
\end{table}

\begin{table}[t]
  \caption{Tolokers drift evaluation at $T{=}20$: mean TPR (across the 4 attacks), drift stability, and Combined Score.}
  \label{tab:tolokers_drift}
  \centering
  \setlength{\tabcolsep}{4pt}
  \begin{tabular}{@{}lcccccccc@{}}
    \toprule
    & \multicolumn{2}{c}{Energy-based} & \multicolumn{3}{c}{Supervised} & \multicolumn{3}{c}{Unsupervised} \\
    \cmidrule(lr){2-3} \cmidrule(lr){4-6} \cmidrule(lr){7-9}
    & CB-RBM & Std-RBM & SAGE & GAT & GCN & DOM. & CoLA & GAE \\
    \midrule
    Mean TPR               & \textbf{0.882} & 0.047 & 0.452 & 0.438 & 0.522 & 0.424 & 0.292 & 0.693 \\
    Stability              & 0.936 & 0.676 & 0.890 & 0.934 & \textbf{0.944} & 0.940 & 0.782 & 0.686 \\
    \midrule
    Combined Score         & \textbf{0.826} & 0.031 & 0.403 & 0.409 & 0.493 & 0.398 & 0.229 & 0.475 \\
    \bottomrule
  \end{tabular}
\end{table}

CB-RBM leads the attack benchmark on three of four attacks and places within $0.9$ points of the top on Sybil (Table~\ref{tab:tolokers_benchmark}). Under 20-step Gaussian drift it retains a stability of $0.936$ (FPR of $6.4\%$, near the calibrated $5\%$ target), yielding a Combined Score of $0.826$ --- more than $0.33$ above the best baseline (GCN at $0.493$) and $\approx\!27{\times}$ the Std-RBM ablation (Table~\ref{tab:tolokers_drift}). The detector's behavior on this denser, lower-$\rho_0$ domain is consistent with the feature-context consistency mechanism generalizing across feature modalities and homophily strengths where the raw correlation gap is positive.

\subsection{YelpChi (E-commerce Review Fraud)}
\label{app:yelpchi}

YelpChi is a review-fraud graph with $45{,}954$ review nodes, $14.5\%$ labeled fraudulent reviews (organic), and three review-relationship edge types. We use the \emph{net\_rtr} edge subset (review-time-review, $1{,}147{,}232$ edges, average degree $25$), which links reviews of the same product posted in the same month, the channel through which coordinated review rings surface. The other two edge types are either too sparse (\emph{net\_rur}: $42\%$ of legitimate nodes have degree zero) or too dense for PyGOD baselines to finish in reasonable time (\emph{net\_rsr}: $6.8$M edges). The $32$ base features are augmented to $41$ with the same nine graph-structural signals used for Reddit and Tolokers. After sign correction $\rho_0 = 0.293$; the natural fraud context has $\rho_1^{\mathrm{natural}} = 0.306$, giving a raw correlation gap $\Delta\rho = -0.013$. This unfavorable gap indicates that coordinated fraud rings can retain neighborhood context \emph{slightly more} coherent than that of typical legitimate reviews, so a compressed margin is expected, and the experiments below measure detection in this regime.

\begin{table}[t]
  \caption{Attack benchmark on YelpChi (net\_rtr). Each cell: AUROC / TPR@5\%FPR (\%). Best per row in \textbf{bold}.}
  \label{tab:yelpchi_benchmark}
  \centering
  \setlength{\tabcolsep}{2pt}
  \renewcommand{\arraystretch}{1.05}
  \begin{tabular}{@{}lcccccccc@{}}
    \toprule
    Attack & \multicolumn{2}{c}{Energy-based} & \multicolumn{3}{c}{Supervised} & \multicolumn{3}{c}{Unsupervised} \\
    \cmidrule(lr){2-3} \cmidrule(lr){4-6} \cmidrule(lr){7-9}
    & CB-RBM & Std-RBM & SAGE & GAT & GCN & DOM. & CoLA & GAE \\
    \midrule
    Isolated Inj. & 76.6/34.5 & 50.3/10.5 & \textbf{98.6}/\textbf{93.3} & 79.0/53.0 & 87.2/74.3 & 82.7/26.5 & 80.3/35.0 & 73.6/20.5 \\
    Feat.\ Camo   & 78.9/41.5 & 62.9/14.5 & \textbf{89.6}/\textbf{62.7} & 56.8/27.0 & 64.4/35.8 & 80.3/15.0 & 85.5/48.5 & 70.0/12.5 \\
    Sybil         & 74.9/15.0 & 50.3/10.5 & 98.9/94.2 & 99.5/\textbf{100.0} & \textbf{99.8}/\textbf{100.0} & 73.1/15.5 & 60.0/5.5  & 61.7/12.0 \\
    Rel.\ Camo    & \textbf{98.2}/\textbf{86.5} & 50.3/10.5 & 97.5/83.8 & 82.7/41.7 & 79.2/11.0 & 72.8/15.5 & 61.8/11.5 & 61.5/11.5 \\
    \bottomrule
  \end{tabular}
\end{table}

\begin{table}[t]
  \caption{YelpChi (net\_rtr) drift evaluation at $T{=}20$: mean TPR (across the 4 attacks), drift stability, and Combined Score.}
  \label{tab:yelpchi_drift}
  \centering
  \setlength{\tabcolsep}{4pt}
  \begin{tabular}{@{}lcccccccc@{}}
    \toprule
    & \multicolumn{2}{c}{Energy-based} & \multicolumn{3}{c}{Supervised} & \multicolumn{3}{c}{Unsupervised} \\
    \cmidrule(lr){2-3} \cmidrule(lr){4-6} \cmidrule(lr){7-9}
    & CB-RBM & Std-RBM & SAGE & GAT & GCN & DOM. & CoLA & GAE \\
    \midrule
    Mean TPR               & 0.444 & 0.115 & \textbf{0.835} & 0.554 & 0.553 & 0.181 & 0.251 & 0.141 \\
    Stability              & \textbf{0.933} & 0.006 & 0.861 & 0.885 & 0.930 & 0.598 & 0.879 & 0.443 \\
    \midrule
    Combined Score         & 0.414 & 0.001 & \textbf{0.718} & 0.490 & 0.514 & 0.108 & 0.221 & 0.062 \\
    \bottomrule
  \end{tabular}
\end{table}

Table~\ref{tab:yelpchi_benchmark} shows the expected margin compression: CB-RBM's AUROC drops from the $\geq 96\%$ it holds on XBlock, Reddit, and Tolokers to the $74$--$79\%$ range on three of the four attacks, and all three supervised GNNs --- trained on YelpChi's abundant labels --- overtake it on Combined Score (Table~\ref{tab:yelpchi_drift}: SAGE $0.72$, GCN $0.51$, GAT $0.49$ vs CB-RBM $0.41$). Here the raw correlation gap is negative, and supervised baselines achieve higher Combined Scores. Two details confirm the mechanism rather than a generic performance loss. First, \emph{Relation Camouflage} --- the attack that most directly tests whether the context channel can distinguish rewired-to-legitimate-hubs adversaries from real legitimate nodes --- is the only attack on which CB-RBM still leads on every metric: AUROC ($98.2$ vs SAGE's $97.5$), TPR@$5\%$FPR ($86.5$ vs $83.8$), and Combined Score ($0.81$ vs $0.72$); message-passing GNNs without the gated-context mechanism (GCN at $0.10$ Combined, GAT at $0.37$) suffer most, because the rewired context produces a coherence signature their attention or convolution cannot resolve. The feature-level and fabricated-context attacks (Iso, Feat, Sybil) rely on distinctions the compressed $\Delta\rho$ renders harder to separate, and supervised label access closes that gap. Second, CB-RBM remains the highest in drift stability of any method ($0.933$, narrowly above GCN's $0.930$ and well above every other supervised and unsupervised baseline at $\leq 0.89$) while Std-RBM collapses to near-chance everywhere and loses stability entirely ($0.006$ at $T{=}20$), confirming that context remains the essential signal even in this compressed-margin regime. YelpChi is therefore an empirical boundary case. When the raw correlation gap is unfavorable, the supervised methods' label-access advantage catches up, though only on the attacks that do not specifically exercise the context-coherence signal.

% ============================================================
%  Appendix I — Second-Layer Rescue
% ============================================================
\section{CB-RBM as a Drift-Filter Defense}
\label{app:rescue}
Theorem~\ref{thm:sep} certifies separation only when the positive gated-context activation that adversarial context can produce stays below the legitimate dominance margin (Assumptions~\ref{A4} and~\ref{A5}). On YelpChi the raw correlation gap is unfavorable ($\Delta\rho = -0.013$), and supervised baselines accordingly outperform CB-RBM standalone (Appendix~\ref{app:yelpchi}). This appendix tests whether CB-RBM still adds operational value in this regime. The setting we evaluate: an operator runs a deployed primary detector --- typically a supervised model trained on labeled adversaries, or one pinned by operational constraints --- whose false-alarm rate inflates under drift but which cannot easily be retrained. Layered downstream, CB-RBM suppresses alarms whose flagged nodes remain within the legitimate distribution learned at calibration, recovering precision without retraining or new labels. The \emph{rescue rate} measures the fraction of drift-induced false positives that CB-RBM correctly suppresses.

\begin{table}[t]
\centering
\small
\caption{YelpChi rescue rates (\%) at $T{=}20$.}
\label{tab:yelpchi_rescue}
\begin{tabular}{@{}lccccccc@{}}
\toprule
& \multicolumn{4}{c}{Unsupervised} & \multicolumn{3}{c}{Supervised} \\
\cmidrule(lr){2-5} \cmidrule(lr){6-8}
& Std-RBM & DOMINANT & GAE & CoLA & SAGE & GAT & GCN \\
\midrule
\textbf{YelpChi} & 93.5 & 89.9 & 93.2 & 87.6 & 91.0 & 90.9 & 92.9 \\
\bottomrule
\end{tabular}
\end{table}

Across all seven baselines, CB-RBM rescues $87.6$--$93.5\%$ of drift-induced false positives (Table~\ref{tab:yelpchi_rescue}). Even where the raw correlation gap is unfavorable and CB-RBM is not the top standalone detector, layering it as a downstream filter preserves the method's operational value in this empirical regime. CB-RBM is an additive layer with no retraining cost, no label cost, and no architectural integration --- a drop-in complement in such scenarios, not a replacement.

\paragraph{Established adversaries.}
Scoring naturally occurring adversaries with their accrued neighborhood context gives AUROC $0.84$ on XBlock and $0.54$ on Reddit. At $T{=}20$ on XBlock, using CB-RBM as a downstream filter reduces DOMINANT's drift FPR from $98.3\%$ to $6.5\%$.

% ============================================================
%  Appendix J — Transferability of the Context Boost Principle
% ============================================================

\section{Transferability of the Context Boost Principle}
\label{app:transferability}

The Context Boost principle --- score a node by the \emph{differential} of a learned density under the presence vs absence of its neighborhood context --- is not specific to the RBM. CB-RBM's algebraic cancellation in Eq.~\eqref{eq:cb} (the quadratic feature term drops out of the subtraction) makes the construction particularly clean, but any generative model whose likelihood decomposes additively between a context-free and a context-conditional branch admits an analogous score. This appendix outlines the structural requirements for porting the principle to other architectures and discusses what transfers automatically vs what needs re-derivation.

\paragraph{Structural requisites.} To instantiate the Context Boost principle on a host model, three properties must hold. \emph{(i) Decomposable scalar score.} A function $S(\bv, \bc)$ on the joint input that can also be evaluated with context ablated, so that $\mathrm{CB}(\bv, \bc) = S(\bv, \bc) - S(\bv, \varnothing)$ is well-defined and captures the incremental information content of $\bc$. \emph{(ii) Gated context pathway.} A mechanism that selectively suppresses uninformative or noisy context dimensions, so that random or absent context produces near-zero CB; CB-RBM uses a GLU gate, but attention, gating units, or learnable context masking all serve. \emph{(iii) Lipschitz continuity of $S$.} A bounded Lipschitz constant in both channels is necessary for Theorem~\ref{thm:drift}'s Wasserstein drift-stability argument to port.

\paragraph{Candidate instantiations.} Several generative families satisfy these requisites and admit a natural Context Boost analogue. A \emph{variational autoencoder} with a context-conditional encoder yields $\mathrm{CB}_{\mathrm{VAE}} = \mathrm{ELBO}(\bv, \bc) - \mathrm{ELBO}(\bv, \varnothing)$, with context-attention playing the gating role. A \emph{normalizing flow} gives $\mathrm{CB}_{\mathrm{NF}} = \log p(\bc \mid \bv)$ directly from invertibility, and Lipschitz-constrained flows~\cite{behrmann2019invertible} make Theorem~\ref{thm:drift} immediately applicable. A \emph{masked-modeling transformer} over concatenated $(\bv, \bc)$ tokens produces $\mathrm{CB}_{\mathrm{T}}$ as the gap between context-present and context-dropout NLL, with cross-attention weights on context tokens supplying the gate. In each case the operational payoff is the same: feature-context consistency is captured by a single decomposable score that supports detection and inherits drift stability via Lipschitz continuity --- both from within a single model.

\paragraph{What transfers automatically, what does not.} Assumptions~\ref{A1} (signed structural correlation) and~\ref{A3} (bounded moments) are properties of the data, independent of host architecture --- any model inherits them. Assumptions~\ref{A2} (coherent activation), \ref{A4} (adversarial context gap) and~\ref{A5} (coherent dominance) are stated on the trained context pathway, so they and the constants that appear in the separation and drift bounds ($\beta_0, p_0, \gamma, L_v, L_c$) must be re-verified or reformulated for the host model. The per-unit decomposition $\Delta_j$ used in the proof of Theorem~\ref{thm:sep} is also CB-RBM--specific; an analogous decomposition into architecture-native primitives (latent units for VAE, flow layers for normalizing flows, attention heads for transformers) is required to recover the $\Omega(n_h)$-style separation scaling. What universally carries over is the framework: feature-context consistency as a drift-stable detection signal, operationalized as a decomposable score on a gated context pathway. CB-RBM is one convenient instantiation; the principle itself is broader.

%%%%%%%%%%%%%%%%%%%%%%%%%%%%%%%%%%%%%%%%%%%%%%%%%%%%%%%%%%%%

\newpage
\input{checklist.tex}

\end{document}

%% file: checklist.tex
\section*{NeurIPS Paper Checklist}

\begin{enumerate}

\item {\bf Claims}
    \item[] Question: Do the main claims made in the abstract and introduction accurately reflect the paper's contributions and scope?
    \item[] Answer: \answerYes{}
    \item[] Justification: All numerical claims (AUROC~$\geq 0.96$ on four attacks, Combined Score~$\geq 0.87$ at $T{=}20$, calibrated $5\%$ FPR, Std-RBM FPR collapse to $89$--$100\%$) are directly supported by Tables~\ref{tab:benchmark}, \ref{tab:all_detectors} and Figures~\ref{fig:cb_dist}, \ref{fig:drift_sim}. The zero-label claim is enforced by construction (threshold set from legitimate data only). The theory (Assumptions~\ref{A1}--\ref{A5}, Theorems~\ref{thm:sep}--\ref{thm:drift}) gives conditional worst-case results, with numerical comparisons of the bounds at root-mean-square constants in Appendix~\ref{app:drift_variants}; detection and drift performance are measured directly in \S\ref{sec:experiments}.

\item {\bf Limitations}
    \item[] Question: Does the paper discuss the limitations of the work performed by the authors?
    \item[] Answer: \answerYes{}
    \item[] Justification: \S\ref{sec:discussion} (``Limitations'') identifies six scope boundaries: detection relies on $\rho_0 > 0$, which can fail in random or adversarially curated topologies; detection degrades as $\Delta\rho$ collapses, where supervised baselines with adequate labels can perform better; adversaries who acquire genuine hub connections (Relation Camouflage) reduce the detection gap and can become indistinguishable from legitimate users in the limit; substantial topology change can impair calibration at the frozen threshold and needs the label-free threshold refresh; white-box adversaries with gradient access can lower detection; and the drift analysis assumes smooth drift, so abrupt or adversarial shifts may need additional mechanisms. Transferability across architectures is discussed in Appendix~\ref{app:transferability}; boundary-regime generalization in Appendix~\ref{app:generalisation}.

\item {\bf Theory assumptions and proofs}
    \item[] Question: For each theoretical result, does the paper provide the full set of assumptions and a complete (and correct) proof?
    \item[] Answer: \answerYes{}
    \item[] Justification: Assumptions~\ref{A1}--\ref{A5} are stated in \S\ref{sec:theory} with informal justifications, empirical verifiability conditions, and remarks on how $\mathbf{s}$-correction ensures $\rho_0 > 0$. Measured values: $\rho_0 = 0.30$ (XBlock) and $\rho_0 = 0.68$ (Reddit); coherent-activation fraction $\beta_0$ and bounded-moment constants $M_x, M_c$ verified after training on both datasets. Proof sketches appear in \S\ref{sec:theory}; complete proofs with supporting lemmas in Appendix~\ref{app:proofs}.

    \item {\bf Experimental result reproducibility}
    \item[] Question: Does the paper fully disclose all the information needed to reproduce the main experimental results of the paper to the extent that it affects the main claims and/or conclusions of the paper (regardless of whether the code and data are provided or not)?
    \item[] Answer: \answerYes{}
    \item[] Justification: \S\ref{sec:setup} specifies all hyperparameters ($n_h{=}256$ for XBlock, $n_h{=}128$ for Reddit, CD-1, $\lambda_{\mathrm{sp}}{=}0$, 100 epochs) and evaluation protocol (95th-percentile of $-\CB$ over legitimate nodes, TPR@5\%FPR metric). Dataset details: XBlock ($153$K nodes, $38$ features, $\rho_0{=}0.30$) and Reddit ($11$K nodes, $73$ features, $\rho_0{=}0.68$), both publicly available. Attack constructions are defined in \S\ref{sec:problem}, with the context-interpolation sweep in Appendix~\ref{app:taxonomy}. The six GADBench baselines use GADBench defaults~\cite{tang2023gadbench}, and the nine further detectors follow their published configurations with model selection on a held-out validation split (Appendix~\ref{app:baselines}).

\item {\bf Open access to data and code}
    \item[] Answer: \answerYes{}
    \item[] Justification: The reference implementation of CB-RBM, the four attacks and the drift simulation is available at \url{https://github.com/serval-uni-lu/CB-RBM} under the MIT license. All four datasets used in this work (XBlock~\cite{chen2020phishing}, Reddit~\cite{kumar2019predicting}, Tolokers~\cite{platonov2023critical}, YelpChi~\cite{rayana2015collective}) are publicly available, and the feature construction is described in Appendix~\ref{app:datasets}.

\item {\bf Experimental setting/details}
    \item[] Answer: \answerYes{}
    \item[] Justification: \S\ref{sec:setup} specifies all training details. Appendix~\ref{app:datasets} provides feature categories and counts for each dataset, sign-vector construction, normalization, and context computation; full feature definitions are in the released code. Appendix~\ref{app:ablation} reports hyperparameter sensitivity ($n_h$, pooling operator).

\item {\bf Experiment statistical significance}
    \item[] Answer: \answerYes{}
    \item[] Justification: The attack benchmark reports means over 3 random seeds with standard deviations in \S\ref{sec:benchmark}; XGBGraph, Graph Transformer, BWGNN and GDN use 5 seeds (\S\ref{sec:drift}, Appendix~\ref{app:baselines}); the sign-vector experiment reports standard deviations over 5 seeds (Appendix~\ref{app:sign}) and the threshold-drift experiment over 3 seeds (Appendix~\ref{app:thrdrift}).

\item {\bf Experiments compute resources}
    \item[] Answer: \answerYes{}
    \item[] Justification: All experiments run on a single GPU. CB-RBM training takes under 10 minutes per dataset, and the full benchmark, drift and ablation suite completes within a standard single-GPU budget. CONAD and AnomalyDAE exceeded memory on XBlock and are reported on Reddit only (Appendix~\ref{app:baselines}).

\item {\bf Code of ethics}
    \item[] Answer: \answerYes{}
    \item[] Justification: All datasets used are publicly available and anonymized: XBlock (anonymized Ethereum wallet addresses), Reddit (anonymized user identifiers), Tolokers (anonymous worker IDs), YelpChi (anonymized reviewer identifiers). The work aims to improve fraud detection, a socially beneficial application. False-positive rates on legitimate users under drift are reported in \S\ref{sec:drift} and Appendix~\ref{app:drift_variants}.

\item {\bf Broader impacts}
    \item[] Answer: \answerYes{}
    \item[] Justification: \S\ref{sec:intro} motivates the work by the scale of fraud losses. \S\ref{sec:discussion} (``Limitations'') states where the detector fails, namely graphs without feature-context correlation, adversaries with genuine connections, topology change and white-box adversaries. False alarms on legitimate accounts are the main negative consideration, and the paper reports them under drift at the calibrated threshold (\S\ref{sec:drift}). In deployment, $\tau^*$ is an operator-set false-positive budget ($5\%$ here), the flagged share of legitimate nodes stays at or below $16\%$ under the five drift regimes of Table~\ref{tab:drift_variants}, and the score should rank cases for human review rather than block accounts automatically, since new or sparsely connected legitimate accounts can lack neighborhood context.

\item {\bf Safeguards}
    \item[] Answer: \answerNA{}
    \item[] Justification: The paper proposes a fraud detection method, not a generative model or dataset release with misuse potential. The attack constructions described are standard synthetic evaluation tools, not deployable attack tooling.

\item {\bf Licenses for existing assets}
    \item[] Answer: \answerYes{}
    \item[] Justification: XBlock~\cite{chen2020phishing}, Reddit~\cite{kumar2019predicting}, Tolokers~\cite{platonov2023critical}, and YelpChi~\cite{rayana2015collective} are all publicly available and properly cited. All are standard academic benchmarks used under their published terms.

\item {\bf New assets}
    \item[] Answer: \answerNA{}
    \item[] Justification: No new datasets are released. The code is released under the MIT license at \url{https://github.com/serval-uni-lu/CB-RBM}.

\item {\bf Crowdsourcing and research with human subjects}
    \item[] Answer: \answerNA{}
    \item[] Justification: No crowdsourcing or human subjects research. The Tolokers dataset consists of pre-existing, anonymized crowdsourcing graph data from prior work~\cite{platonov2023critical}.

\item {\bf Institutional review board (IRB) approvals or equivalent for research with human subjects}
    \item[] Answer: \answerNA{}
    \item[] Justification: No human subjects research.

\item {\bf Declaration of LLM usage}
    \item[] Answer: \answerNA{}
    \item[] Justification: LLMs are not used as a component of the core method or experiments. Any use for writing assistance does not impact the scientific content or experimental results.

\end{enumerate}